\documentclass[twoside,11pt]{article}

\usepackage{blindtext}

\usepackage[abbrvbib]{jmlr2e}

\usepackage{amsmath,amsfonts}
\usepackage{mathtools}
\mathtoolsset{showonlyrefs}
\usepackage{subcaption}
\usepackage{algorithm}
\usepackage{algpseudocode}
\usepackage{bbm}
\usepackage{xcolor}

\DeclareMathOperator*{\argmin}{arg\,min}

\newcommand{\R}{\mathbb{R}}

\newcommand{\E}{\mathbb{E}}

\newcommand\dx{\,\mathrm{d}}
\newcommand\tT{\mathrm{T}}
\newcommand\ELBO{\mathrm{ELBO}}

\makeatletter
\renewcommand{\paragraph}{%
  \@startsection{paragraph}{4}%
  {\z@}{3.25ex \@plus 1ex \@minus .2ex}{-1em}%
  {\normalfont\normalsize\itshape}%
}
\makeatother

\usepackage{lastpage}
\jmlrheading{25}{2024}{1-\pageref{LastPage}}{3/23; Revised
2/24}{6/24}{23-0396}{Giovanni S. Alberti, Johannes Hertrich, Matteo Santacesaria and Silvia Sciutto}

\ShortHeadings{Manifold Learning by Mixture Models of VAEs}{Alberti, Hertrich, Santacesaria and Sciutto}
\firstpageno{1}

\begin{document}

\title{Manifold Learning by Mixture Models of VAEs\\ for Inverse Problems}

\author{\name Giovanni S. Alberti \email giovanni.alberti@unige.it \\
       \addr MaLGa Center\\ 
       Department of Mathematics, Department of Excellence 2023--2027\\
       University of Genoa, Italy
       \AND
       \name Johannes Hertrich \email j.hertrich@ucl.ac.uk \\
       \addr Department of Computer Science\\
       University College London,\\
       London, United Kingdom
       \AND
       \name Matteo Santacesaria \email matteo.santacesaria@unige.it \\
       \addr MaLGa Center\\ 
       Department of Mathematics, Department of Excellence 2023--2027\\
       University of Genoa, Italy
       \AND
       \name Silvia Sciutto \email silvia.sciutto@edu.unige.it \\
       \addr MaLGa Center\\ 
       Department of Mathematics, Department of Excellence 2023--2027\\
       University of Genoa, Italy
       }

\editor{Amos Storkey}

\maketitle

\begin{abstract}
Representing a manifold of very high-dimensional data with generative models has been shown to be computationally efficient in practice.
However, this requires that the data manifold admits a global parameterization.
In order to represent manifolds of arbitrary topology, we propose to learn a mixture model of variational autoencoders.
Here, every encoder-decoder pair represents one chart of a manifold.
We propose a loss function for maximum likelihood estimation of the model weights and choose an architecture that provides
us the analytical expression of the charts and of their inverses.
Once the manifold is learned, we use it for solving inverse problems by minimizing a data fidelity term restricted to the learned manifold.
To solve the arising minimization problem we propose a Riemannian gradient descent algorithm on the learned manifold.
We demonstrate the performance of our method for low-dimensional toy examples as well as for deblurring and 
electrical impedance tomography on certain image manifolds.
\end{abstract}


\section{Introduction}

\paragraph{Manifold learning.}

The treatment of high-dimensional data is often computationally costly and numerically unstable.
Therefore, in many applications, it is important to find a low-dimensional representation of high-dimensional data sets.
Classical methods, like the principal component analysis (PCA, \citealp{P1901}),  assume that the data is contained in a low-dimensional subspace.
However, for complex data sets this assumption appears to be too restrictive, particularly when working with image data sets. 
Therefore, recent methods rely on the so-called manifold hypothesis \citep{BCV2013}, stating that even complex and high-dimensional data sets are 
contained in a low-dimensional manifold.
Based on this hypothesis, in recent years, many successful approaches have been based on generative models, able to represent high dimensional data in $\R^n$ by a generator $D\colon\R^d\to\R^n$ with $d\ll n$: these include generative adversarial networks (GANs, \citealp{GPMX2014}),  
variational autoencoders (VAEs, \citealp{KW2013}), injective flows \citep{KKHD2021} and score-based diffusion models \citep{song2019generative,ho2020denoising}. For a survey on older approaches to manifold learning, the reader is referred to \citet{ma2011manifold,2012-izenman} and to the references therein.

\paragraph{Learning manifolds with multiple charts.} Under the assumption that $D$ is injective, the set of generated points $\{D(z):z\in\R^d\}$ forms a manifold that approximates the training set.
However, this requires that the data manifold admits a global parameterization. 
In particular, it must not be disconnected or contain holes.
In order to model disconnected manifolds, \citet{FZWN2021,JZTT2017,KRRA2022,PL2018} propose to model the latent space of a VAE by a Gaussian mixture model.
This enables the authors to capture multimodal probability distributions.
However, this approach struggles with modelling manifolds with holes since either the injectivity of the generator is violated or it is impossible to model overlapping charts.
Similarly, \citet{DFDKT2018,MLMTT2019,RMP2020} propose latent distributions defined on Riemannian manifolds for representing general topologies. \citet{2022-massa-garbarino-benvenuto} embed the manifold into a higher-dimensional space, in the spirit of Whitney embedding theorem.
However, these approaches have the drawback that the topology of the manifold has to be known a priori, which is usually not the case in practice.

Here, we focus on the representation of the data manifold by several charts.
A chart provides a parameterization of an open subset from the manifold by defining a mapping from the manifold into a Euclidean space.
Then, the manifold is represented by the collection of all of these charts, which is called atlas.
For finding these charts, \citet{CDJ2021,FG2021,PRA2013,SDJBS2022} propose the use of clustering algorithms.
By default, these methods do not provide an explicit formulation of the resulting charts.
As a remedy, \citet{B2002,PRA2013} use linear or kernelized embeddings.
\citet{SDJBS2022} propose to learn for each chart again a generative model.
However, these approaches often require a large number of charts and are limited to relatively low data dimensions.
The idea of representing the charts by generative models is further elaborated by \citet{K2018,K2021,SCL2019}.
Here, the authors proposes to train at the same time several (non-variational) autoencoders and a classification network that decides
for each point to which chart it belongs.
In contrast to the clustering-based algorithms, the computational effort scales well for large data dimensions.
On the other hand, the numerical examples in the corresponding papers show that the approach already has difficulties to approximate small toy examples like a torus.

In this paper, we propose to approximate the data manifold by a mixture model of VAEs. 
Using Bayes theorem and the ELBO approximation of the likelihood term we derive a loss function for maximum likelihood estimation of the model weights.
Mixture models of generative models for modelling disconnected data sets were already considered by \citet{BGP2017,HNLP2017,LVTRGS2018,SS2022,YB2021}.
However, they are trained in a different way and to the best of our knowledge none of those is used for manifold learning.

\paragraph{Inverse Problems on Manifolds.}

Many problems in applied mathematics and image processing can be formulated as inverse problems.
Here, we consider an observation $y$ which is generated by
\begin{equation}\label{eq:inverse}
y=\mathcal G(x)+\eta,
\end{equation}
where $\mathcal G\colon\R^n\to\R^m$ is an ill-posed or ill-conditioned, possibly nonlinear forward operator and $\eta$ represents additive noise.
Reconstructing the input $x$ directly from the observation $y$ is usually not possible due to the ill-posed operator and the high dimension of the problem.
As a remedy, the incorporation of prior knowledge is required. 
This is usually achieved by using regularization theory, namely, by minimizing the sum of a data fidelity term $F(x)$ and a regularizer $R(x)$, where $F$ describes the fit of $x$ to $y$ and $R$ incorporates
the prior knowledge.
With the success of deep learning, data-driven regularizers became popular \citep{ADHHMS2022,AMOS2019,GNBDU2022,HHR2022,LOS2018}.

In this paper, we consider a regularizer which constraints the reconstruction $x$ to a learned data manifold $\mathcal M$.
More precisely, we consider the optimization problem
$$
\hat x=\argmin_x F(x)\quad \text{subject to $x\in\mathcal M$,}
$$
where $F(x)=\frac12\|\mathcal G(x)-y\|^2$ is a data-fidelity term.
This corresponds to the regularizer $R(x)$ which is zero for $x\in\mathcal M$ and infinity otherwise.
When the manifold admits a global parameterization given by one single generator $D$, \citet{ASS2022,CKFB2020,DCE2021,GAT2022} propose to 
reformulate the problem as $\hat x=\hat z$, where $\hat z=\argmin_z F(D(x))$.
Since this is an unconstrained problem, it can be solved by gradient based methods.
However, since we consider manifolds represented by several charts, this reformulation cannot be applied.
As a remedy, we propose to use a Riemannian gradient descent scheme.
In particular, we derive the Riemannian gradient using the decoders and encoders of our manifold
and propose two suitable retractions for applying a descent step into the gradient direction.

To emphasize the advantage of using multiple generators, we demonstrate the performance of our method on numerical examples.
We first consider some two- and three-dimensional toy examples.
Finally, we apply our method to deblurring and to electrical impedance tomography (EIT), a nonlinear inverse problem consisting in the reconstruction of the leading coefficient of a second order elliptic PDE from the knowledge of the boundary values of its solutions \citep{cheney1999electrical}.
The code of the numerical examples is available online.\footnote{The code is available at \url{https://github.com/johertrich/Manifold_Mixture_VAEs}}

\paragraph{Outline.}

The paper is organized as follows. In Section~\ref{sec:VAEs}, we revisit VAEs and fix the corresponding notations.
Afterwards, in Section~\ref{sec:chart_learning}, we introduce mixture models of VAEs for learning embedded manifolds of arbitrary dimensions and topologies. Here, we focus particularly on the derivation of the loss function and of the architecture, which allows us to access the charts and their inverses.
For minimizing functions defined on the learned manifold, we propose a Riemannian gradient descent scheme in Section~\ref{sec:opt_prob}.
We provide numerical toy examples for one and two dimensional manifolds in Section~\ref{sec:toy_examples}.
In Section~\ref{sec:inverse_problems}, we discuss the applications to deblurring and to electrical impedance tomography.
Conclusions are drawn in Section~\ref{sec:conclusions}.

\section{Background on Variational Autoencoders and Manifolds}

In this section, we revisit the technical background of the paper. First, we recall the concept of VAEs, their training procedure and of a learned latent space with normalizing flows. Afterwards, we 
give a short literature review on manifold learning with VAEs, and some basic notions from differential geometry.

\subsection{Variational Autoencoders for Manifold Learning}\label{sec:VAEs}

In this paper, we assume that we are given data points $x_1,\dots,x_N\in\R^n$ for a large dimension $n$.
In order to reduce the computational effort and to regularize inverse problems, we assume that these data-points are located in a lower-dimensional manifold.
We aim to learn the underlying manifold from the data points $x_1,\dots,x_N$ with a VAE \citep{KW2013,KW2019}.

A VAE aims to approximate the underlying high-dimensional probability distribution $P_X$ of the random variable $X$ 
with a lower-dimensional latent random variable $Z\sim P_Z$ on $\R^d$ with $d<n$, by using the data points $x_1,\dots,x_N$.
To this end, we define a decoder $D\colon\R^d\to\R^n$ and an encoder $E\colon\R^n\to\R^d$. The decoder approximates $P_X$ by the distribution $P_{\tilde X}$ of a random variable $\tilde X\coloneqq D(Z)+\eta$, where $\eta\sim\mathcal N(0,\sigma_x^2 I_n)$.
Vice versa, the encoder approximates $P_Z$ from $P_X$ by the distribution $P_{\tilde Z}$ of the random variable $\tilde Z\coloneqq E(X)+\xi$ with $\xi\sim\mathcal N(0,\sigma_z^2 I_d)$.
Now, decoder and encoder are trained such that we have $P_X\approx P_{\tilde X}$ and $P_Z\approx P_{\tilde Z}$.
To this end, we aim to maximize the log-likelihood function
$
\ell(\theta)=\sum_{i=1}^N\log(p_{\tilde X}(x_i)),
$
where $\theta$ denotes the parameters  $D$ and $E$ depend upon.

The log-density $\log(p_{\tilde X}(x))$ induced by the model is called the evidence. However, for VAEs the 
computation of the evidence is intractable. Therefore, \citet{KW2013} suggest to approximate it by the \emph{evidence lower bound} given by
\begin{equation}\label{eq:ELBO_}
\ELBO(x|\theta)\coloneqq \E_{\xi\sim\mathcal N(0,I_d)}[\log(p_Z(E(x)+\sigma_z\xi))-\tfrac1{2\sigma_x^2}\|D(E(x)+\sigma_z\xi)-x\|^2].
\end{equation}
For the sake of completeness, we include its derivation in Appendix~\ref{app:ELBO}.
Finally, a VAE is trained by minimizing the loss function which sums up the negative ELBO values of all data points, i.e.,
$$
\mathcal L_\mathrm{VAE}(\theta)=-\sum_{i=1}^N\ELBO(x_i|\theta).
$$

\paragraph{Learned Latent Space.}
It is a known issue of VAEs that the inferred probability distribution is often more blurry than the ground truth distribution of the data.
A detailed discussion of this issue can be found in Section 2.8.2 of the survey paper by \citet{KW2019}.
As a remedy, the authors suggest to choose a more flexible model.
One possibility is to combine VAEs with normalizing flows, as proposed by \citet{RM2015} or \citet{DW2019}.
Following these approaches, we increase the flexibility of the model by using a latent space learned by a normalizing flow.
The idea is based on the observation that transforming probability distributions in low-dimensional spaces is much cheaper than in high-dimensional spaces. 
Consequently, modelling the low-dimensional latent space can be more effective than learning probability transformations in the high-dimensional data space.
Here, we employ the specific loss function proposed by \citet{HHS2022,HHS2023} for training the arising model.
More precisely, we choose the latent distribution
$$
P_Z={\mathcal T}_\#P_{\Xi},
$$
where $\mathcal T\colon\R^d\to\R^d$ is an invertible neural network, called normalizing flow. In this way, $P_Z$ is the push-forward of a fixed (known) distribution $P_{\Xi}$.
Then, the density $p_{Z}$ is given by
$$
p_{Z}(z)=p_{\Xi}(\mathcal T^{-1}(z))|\mathrm{det}(\nabla \mathcal T^{-1}(z))|.
$$
The parameters of $\mathcal T$ are considered as trainable parameters.
Then, the ELBO reads as
\begin{equation}\label{eq:ELBO}
\begin{aligned}
\ELBO(x|\theta)&\coloneqq \E_{\xi\sim\mathcal N(0,I_d)}[\log(p_{\Xi}(\mathcal T^{-1}(E(x)+\sigma_z\xi)))\\
&+\log(|\mathrm{det}(\nabla \mathcal T^{-1}(E(x)+\sigma_z\xi))|)-\tfrac1{2\sigma_x^2}\|D(E(x)+\sigma_z\xi)-x\|^2],
\end{aligned}
\end{equation}
where $\theta$ are the parameters of the decoder, the encoder and of the normalizing flow $\mathcal T$.

In the literature, there exist several invertible neural network architectures based on coupling blocks \citep{DSB2016,KD2018}, residual networks \citep{BGCD2019,CBDJ2029,H2022}, ODE representations \citep{CRBD2018,GCBS2018,OFLR2021} and autoregressive flows \citep{HKLC2018}.
In our numerics, we use the coupling-based architecture proposed by~\citet{AKWR2018}.

\paragraph{Manifold Learning with VAEs.} In order to obtain a lower-dimensional representation of the data points, some papers propose to approximate the
data-manifold by $\mathcal M\coloneqq\{D(z):z\in\R^d\}$ (see, e.g., \citealp{ASS2022,CKFB2020,DCE2021,GAT2022}). 
However, this is only possible if the data-manifold admits a global parameterization, i.e., it can be approximated by one
generating function. 
This assumption is often violated in practice.
As a toy example, consider the one-dimensional manifold embedded in $\R^2$ that consists of two circles, see Figure~\ref{fig:toy_ex_data}. 
This manifold is disconnected and contains ``holes''. 
Consequently, the topologies of the manifold and of the latent space $\R$ do not coincide,  so that the manifold cannot be approximated by a VAE.
Indeed, this can be verified numerically. 
When we learn a VAE for approximating samples from this manifold, we observe that the two (generated) circles are not closed and that both components are connected, see Figure~\ref{fig:one_gen}.
As a remedy, in the next section, we propose the use of multiple generators to resolve this problem, see Figure~\ref{fig:four_gen}.
For this purpose, we need the notion of charts and atlases.

\begin{figure}
\begin{subfigure}[t]{0.31\textwidth}
\includegraphics[width=\textwidth]{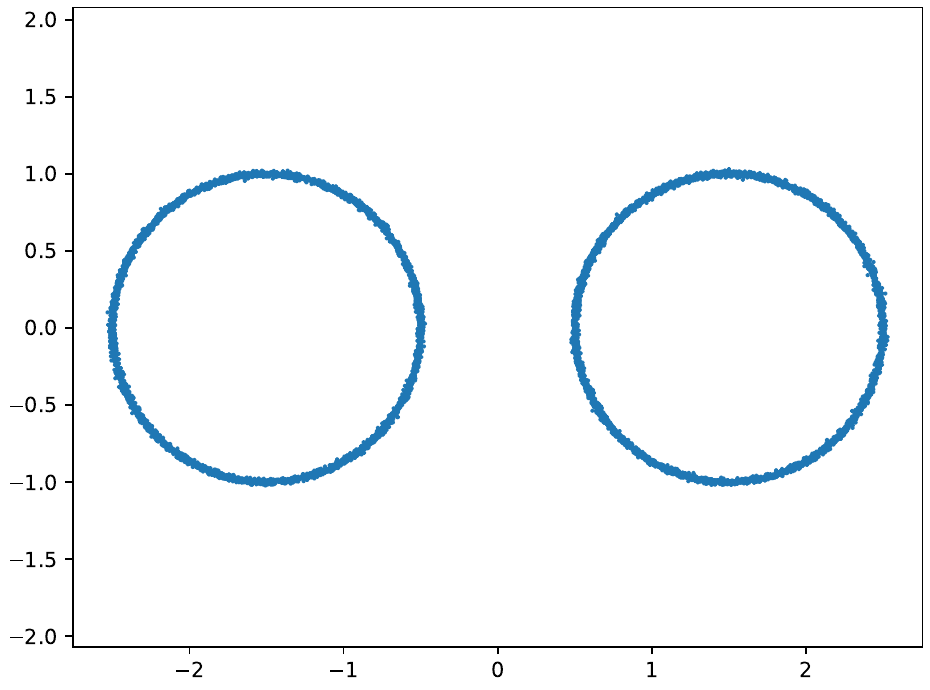}
\caption{Noisy samples from the manifold.}
\label{fig:toy_ex_data}
\end{subfigure}\hfill
\begin{subfigure}[t]{0.31\textwidth}
\includegraphics[width=\textwidth]{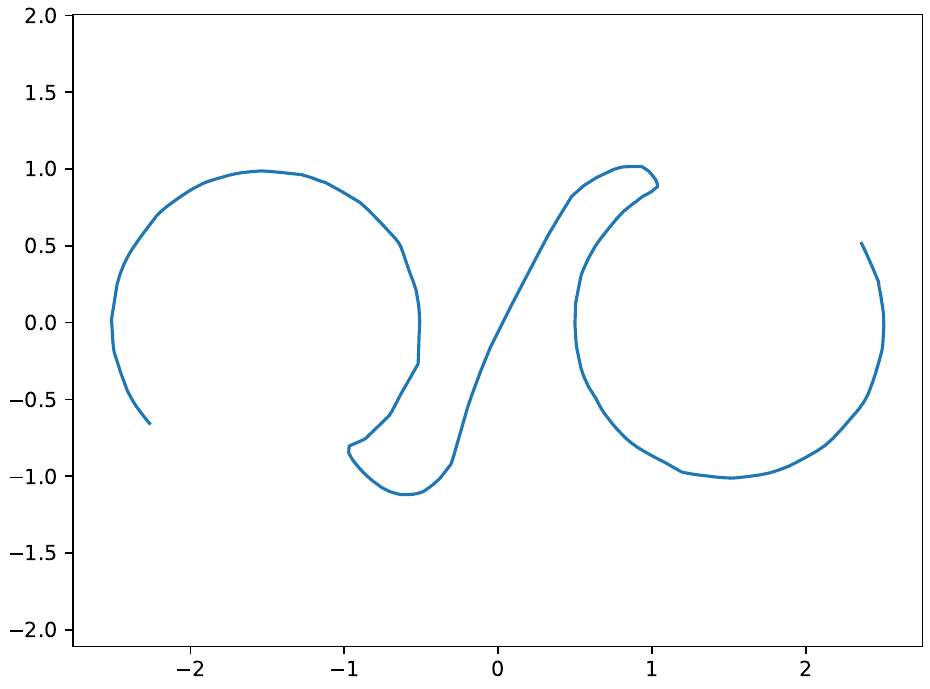}
\caption{Approximation with one chart.}
\label{fig:one_gen}
\end{subfigure}\hfill
\begin{subfigure}[t]{0.31\textwidth}
\includegraphics[width=\textwidth]{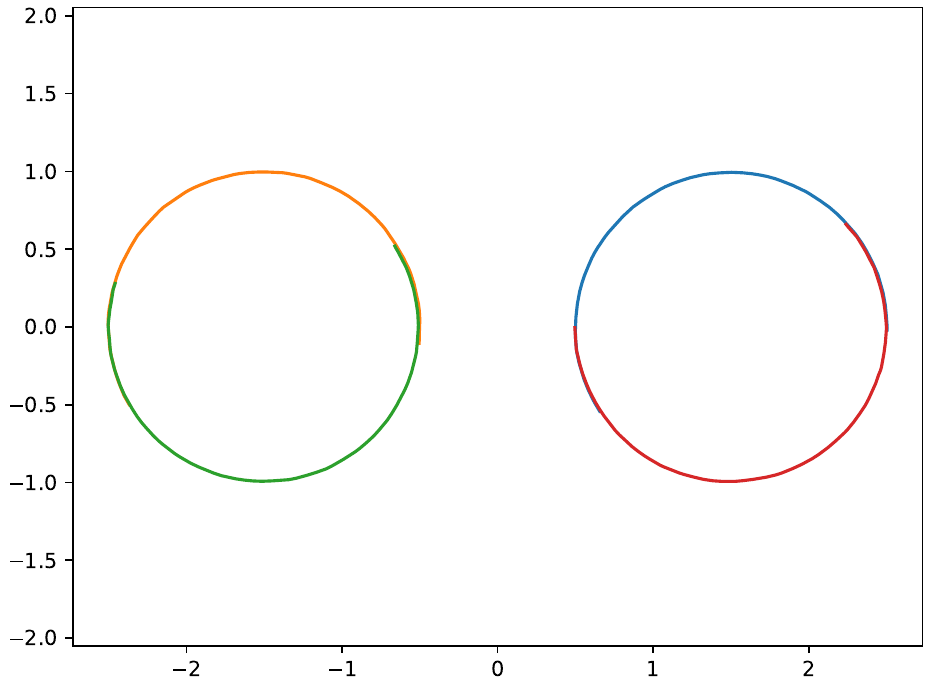}
\caption{Approximation with four charts.}
\label{fig:four_gen}
\end{subfigure}
\caption{Example of a one-dimensional manifold that admits no global parameterization.}
\end{figure}

\subsection{Embedded Manifolds}

A subset $\mathcal M\subseteq \R^n$ is called a $d$-dimensional embedded differentiable  manifold if there exists a family $(U_k,\varphi_k)_{k\in I}$ of relatively open sets $U_k\subseteq \mathcal M$ with $\bigcup_{k\in I} U_k=\mathcal M$ and mappings $\varphi_k\colon U_k\to\R^d$
such that for every $k,l\in I$ 
\begin{itemize}
\item[-] $\varphi_k$ is a homeomorphism between $U_k$ and $\varphi_k(U_k)$;
\item[-] the inverse $\varphi_k^{-1}\colon\varphi_k(U_k)\to U_k$ is continuously differentiable;
\item[-] the transition map $\varphi_k\circ\varphi_l^{-1}\colon \varphi_l(U_k\cap U_l)\to\varphi_k(U_k\cap U_l)$ is continuously differentiable;
\item[-] and the Jacobian $\nabla \varphi_k^{-1}(x)$ of $\varphi_k^{-1}$ at $x$ has full column-rank for any $x\in\varphi_k(U_k)$.
\end{itemize}
We call the mappings $\varphi_k$ charts and the family $(U_k,\varphi_k)_{k\in I}$ an atlas.
With an abuse of notation, we sometimes also call the set $U_k$ or the pair $(U_k,\varphi_k)$ a chart.
 Every compact manifold admits an atlas with finitely many charts $(U_k,\varphi_k)_{k=1}^K$, by definition of compactness.

\section{Chart Learning by Mixtures of VAEs}
\label{sec:chart_learning}

In order to approximate (embedded) manifolds with arbitrary (unknown) topology, we propose to learn several local parameterizations of the manifold instead of a global one. 
To this end, we propose to use mixture models of VAEs.

\paragraph{An Atlas as Mixture of VAEs.}

In this paper, we propose to learn the atlas of an embedded manifold $\mathcal M$ by representing it as a mixture model of variational autoencoders with decoders $D_k\colon\R^d\to\R^n$, encoders $E_k\colon\R^n\to\R^d$ and normalizing flows $\mathcal T_k$ in the latent space, for $k=1,\dots,K$.
Then, the inverse of each chart $\varphi_k$ will be represented by  $\varphi_k^{-1}=\mathcal D_k\coloneqq D_k\circ \mathcal T_k$. Similarly, the chart $\varphi_k$ itself is represented by the mapping $\mathcal E_k\coloneqq \mathcal T_k^{-1}\circ E_k$ restricted to the manifold.
Throughout this paper, we denote the parameters of $(D_k,E_k,\mathcal T_k)$ by $\theta_k$.
Now, let $\tilde X_k$, $k=1,\dots,K$, be the random variable generated by the decoder $D_k$ as in the previous section.
Then, we approximate the distribution $P_X$ of the noisy samples from the manifold by the random variable 
$\tilde X\coloneqq\tilde X_J$, where $J$ is a discrete random variable on $\{1,\dots,K\}$ with $P(J=k)=\alpha_k$ with mixing weights $\alpha_k>0$ fulfilling $\sum_{k=1}^K\alpha_k=1$.

\subsection{Training of Mixtures of VAEs}

\paragraph{Loss function.}
Let $x_1,\dots,x_N$ be the noisy training samples. We initialize the weights $\alpha$ by $\alpha_k=\frac1K$ for all $k$. They will be estimated later in the training algorithm (see Algorithm~\ref{alg:training}).
In order to train mixtures of VAEs, we again minimize an approximation of an upper bound to the negative log likelihood function $-\sum_{i=1}^N\log(p_{\tilde X}(x_i))$.
To this end, we employ the law of total probability and Jensen's inequality to obtain 
\begin{align}
\log(p_{\tilde X}(x_i))
&=\log\Big(\sum_{k=1}^K\beta_{ik}p_{\tilde X}(x_i)\big) 
\geq\log\Big(\sum_{k=1}^K\beta_{ik}\alpha_k p_{\tilde X_k}(x_i)\Big)\\
&\geq\sum_{k=1}^K \beta_{ik}\big(\log(p_{\tilde X_k}(x_i))+\log(\alpha_k)\big)
=\sum_{k=1}^K \beta_{ik}\log(p_{\tilde X_k}(x_i))-\log(K)
\end{align}
where $\beta_{ik}\coloneqq P(J=k|\tilde X=x_i)$ is the probability that the sample $x_i$ 
was generated by the $k$-th random variable $\tilde X_k$ and we used that $\alpha_k=\frac1K$.
Using the definition of conditional probabilities, we observe that

\begin{equation}\label{eq_true_betas}
\beta_{ik}=P(J=k|\tilde X=x_i)=\frac{P(J=k)p_{\tilde X_k}(x_i)}{p_{\tilde X}(x_i)}=\tfrac{\alpha_k p_{\tilde X_k}(x_i)}{\sum_{j=1}^K \alpha_j p_{\tilde X_j}(x_i)}.
\end{equation}
As the computation of $p_{\tilde X_k}$ is intractable, we replace it by the ELBO \eqref{eq:ELBO}, i.e., we approximate $\beta_{ik}$ by
\begin{equation}\label{eq_approx_betas}
\tilde \beta_{ik}=\tfrac{\alpha_k \exp(\ELBO(x_i|\theta_k))}{\sum_{j=1}^K \alpha_j \exp(\ELBO(x_i|\theta_j))}.
\end{equation}
For the architecture used in our numerical examples in Section~\ref{sec:toy_examples}, we can bound the error introduced by this approximation as in Corollary~\ref{cor:Lipschitz_bound} of Appendix~\ref{app:Lipschitz_bound}.
More precisely, we show that there exists some $L$ such that $\frac1L\beta_{ik}\leq\tilde\beta_{ik}\leq L\beta_{ik}$.
Then, we arrive at the approximation
\begin{equation}\label{eq:nll_approx}
\log(p_{\tilde X}(x_i))\approx \ell(x_i|\Theta)=\sum_{k=1}^K\tfrac{\alpha_k \exp(\ELBO(x_i|\theta_k))}{\sum_{j=1}^K \alpha_j \exp(\ELBO(x_i|\theta_j))} \ELBO(x_i|\theta_k) -\log(K).
\end{equation}
By summing up over all $i$, we approximate the negative log likelihood function by the loss function
$$
\mathcal L(\Theta)=-\sum_{i=1}^N\ell(x_i|\Theta),
$$
in order to train the parameters $\Theta=(\theta_k)_{k=1}^K$ of the mixture of VAEs.
Finally, this loss function is then optimized with a stochastic gradient based optimizer like Adam \citep{KB2014}.

\begin{remark}[Lipschitz regularization]
In order to represent the local structure of the manifold and to stabilize the training, we would like to avoid that two points that are close in the latent distribution have too large a distance in the data space.
This corresponds to regularizing the Lipschitz constant of the decoders $D_k$ and of the normalizing flows $\mathcal T_k$.
More precisely, for some small $\sigma>0$, we add the regularization term 
$$
\mathcal R(\Theta)\coloneqq \frac1{\sigma^2}\sum_{k=1}^K\E_{z\sim P_\Xi,\eta\sim\mathcal N(0,\sigma^2)}\Big[D_k(\mathcal T_k(z))-D_k(\mathcal T_k(z+\eta))\Big]
$$
for the first few epochs of the training. Once the charts roughly capture the local structures of the manifold, we avoid the Lipschitz regularization in order to have the interpretation of the loss function as an approximation of the negative log likelihood of the training points.
\end{remark}

\paragraph{Latent Distribution.}

In order to identify the sets $U_k$ defining the domain of the $k$-th learned chart, we choose a latent distribution that is mostly concentrated in the rectangle $(-1,1)^d$.
Then, we can define the domain $U_k$ of the $k$-th learned chart as the set $U_k\coloneqq \mathcal D_k((-1,1)^d)$. 
Since the charts are supposed to overlap, the density should become small close to the boundary.
To this end, we
choose the distribution $P_{\Xi}$ by using the density $p_{\Xi}(z)\coloneqq \prod_{i=1}^d q(z_i)$, where the density $q$ is up to a multiplicative constant given by
$$
q(z)\propto 
\begin{cases}
1,&|z|<0.8,\\ 
4.8 - 4.75 |z|,&|z|\in[0.8,1],\\
0.05 \exp(-100(|z|-1)),&|z|>1,
\end{cases}
$$
see Figure~\ref{fig:latent} for a plot.

Due to approximation errors and noise, we will have to deal with points $x\in\R^n$ that are not exactly located in one of the sets $U_k$. 
In this case, we cannot be certain to which charts the point $x$ actually belongs.
Therefore, we interpret the conditional probability \eqref{eq_true_betas} as the probability that $x_i$ belongs to the $k$-th chart. Since we cannot compute the $\beta_{ik}$ explicitly, we use the approximations 
$\tilde \beta_{ik}$ from \eqref{eq_approx_betas} instead.

\begin{figure}
\centering
\includegraphics[width=.5\textwidth]{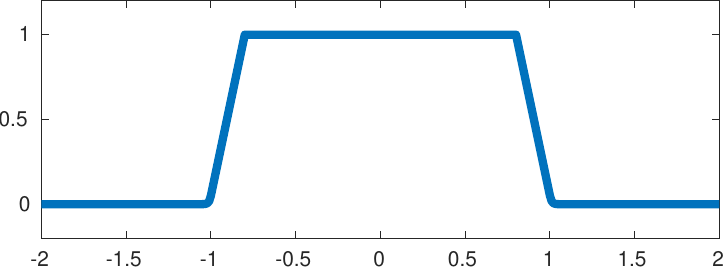}
\caption{Plot of the unnormalized latent density $q$.}
\label{fig:latent}
\end{figure}

\paragraph{Overlapping Charts.}

Since the sets $U_k$ of an atlas $(U_k,\varphi_k)_{k\in I}$ are an open covering of the manifold, they have to overlap near their boundaries.
To this end, we propose the following post-processing heuristic.

By the definition of the loss function $\mathcal L$, we have that the $k$-th generator $\mathcal D_k$ is trained such that $U_k$ contains all points $x_i$ from the training set where $\tilde \beta_{ik}$ is non-zero.
The following procedure  modifies the $\tilde \beta_{ik}$ in such a way that the samples $x$ that are close to the boundary of the $k$-th chart will also be assigned to a second chart.

Since the measure $P_{\Xi}$ is mostly concentrated in $(-1,1)^d$, the region close to the boundary of the $k$-th chart can be identified by $D_k(\mathcal T_k(z))$ for all $z$ close to the boundary of $(-1,1)^d$.
For $c>1$, we define the modified ELBO function
\begin{equation}\label{eq:modELBO}
\begin{aligned}
\ELBO_c(x|\theta_k)&\coloneqq \E_{\xi\sim\mathcal N(0,I_d)}[\log(p_{\Xi}(c \mathcal T_k^{-1}(E_k(x)+\sigma_z\xi)))\\
&+\log(|\mathrm{det}(\nabla \mathcal T_k^{-1}(E_k(x)+\sigma_z\xi))|)-\tfrac1{2\sigma_x^2}\|D_k(E_k(x)+\sigma_z\xi)-x\|^2]
\end{aligned}
\end{equation}
which differs from \eqref{eq:ELBO} by the additional scaling factor $c$ within the first summand.
By construction and by the definiton of $p_\Xi$, it holds that $\ELBO_c(x,\theta_k)\approx\ELBO(x|\theta_k)$ whenever $c\|\mathcal T_k^{-1}(E_k(x))\|_\infty< 0.8$ and $0<\sigma_z\ll0.1$ is small. Otherwise, we have $\ELBO_c(x|\theta_k)< \ELBO(x|\theta_k)$
In particular, $\ELBO_c(x|\theta_k)$ is close to $\ELBO(x|\theta_k)$ if $x$ belongs to the interior of the $k$-th chart and is significantly smaller if it is close to the boundary of the $k$-th chart.

As a variation of $\tilde \beta_{ik}$, now we define 
$$
\hat\gamma_{il}^{(l)}=\frac{\alpha_l\exp(\ELBO_c(x_i|\theta_l))}{\alpha_l\exp(\ELBO_c(x_i|\theta_l))+\sum_{j\in\{1,\dots,K\}\setminus\{l\}} \alpha_j\exp(\ELBO(x_i|\theta_j))}
$$
and
$$
\hat\gamma_{ik}^{(l)}=\frac{\alpha_k\exp(\ELBO(x_i|\theta_k))}{\alpha_l\exp(\ELBO_c(x_i|\theta_l))+\sum_{j\in\{1,\dots,K\}\setminus\{l\}} \alpha_j\exp(\ELBO(x_i|\theta_j))}
$$
for $k,l\in\{1,\dots,K\}$.
Similarly as the $\tilde \beta_{ik}$, $\hat \gamma_{ik}^{(l)}$ can be viewed as a classification parameter, which assigns each $x_i$ either to a chart $k\neq l$ or to the interior part of the $l$-th chart.
Consequently, points located near the boundary of the $l$-th chart will also be assigned to another chart.
Finally, we combine and normalize the $\hat\gamma_{ik}^{(l)}$ by
\begin{equation}\label{eq:gamma}
\gamma_{ik}=\frac{\hat\gamma_{ik}}{\sum_{k=1}^K\hat\gamma_{ik}}, \quad\text{where}\quad\hat \gamma_{ik}=\max_{l=1,\dots,K}\hat\gamma_{ik}^{(l)}.
\end{equation}

Once, the $\gamma_{ik}$ are computed, we update the mixing weights $\alpha$ by $\alpha_k=\sum_{i=1}^N\gamma_{ik}$
and minimize the loss function
\begin{equation}\label{eq:loss_k}
\mathcal L_\text{overlap}(\Theta)=-\sum_{i=1}^N\sum_{k=1}^K \gamma_{ik}\mathrm{ELBO}(x_i|\theta_k)
\end{equation}
using a certain number of epochs of the Adam optimizer \citep{KB2014}.
\\

The whole training procedure for a mixture of VAEs representing the charts of an embedded manifold is summarized in Algorithm~\ref{alg:training}.
The hyperparameters $M_1$, $M_2$ and $M_3$ are chosen large enough such that we have approximately approached a local minimum of the corresponding objective function. In our numerical examples, we choose $M_1=50$, $M_2=150$ and $M_3=50$.

\begin{remark}[Number  of charts $K$]
The choice of the number of charts $K$ is a trade-off between the computational efficiency and the flexibility of the model.
While each manifold has a minimal number of required charts, it could be easily represented by more charts.
Therefore, from a topological viewpoint, there is no upper limit on $K$.
However, since each chart comes with its own pair of decoder and encoder, the number of parameters within the mixture of VAEs, and consequently the training effort, grow with $K$.
\end{remark}

\begin{algorithm}
\begin{algorithmic}
\State 1. Run the Adam optimizer on $\mathcal L(\Theta)+\lambda \mathcal R(\Theta)$ for $M_1$ epochs.
\State 2. Run the Adam optimizer on $\mathcal L(\Theta)$ for $M_2$ epochs.
\State 3. Compute the values $\gamma_{ik}$, $i=1,\dots,N$, $k=1,\dots,K$ from \eqref{eq:gamma}.
\State 4. Compute the mixing weights $\alpha_k=\sum_{i=1}^N\gamma_{ik}$.
\State 5. Run the Adam optimizer on $\mathcal L_\text{overlap}(\Theta)$ from \eqref{eq:loss_k} for $M_3$ epochs.
\end{algorithmic}
\caption{Training procedure for mixtures of VAEs.}
\label{alg:training}
\end{algorithm}

\subsection{Architectures}\label{sec:architectures}

In this subsection, we focus on the architecture of the VAEs used in the mixture model representing the manifold $\mathcal M$.
Since $\mathcal D_k =D_k\circ \mathcal T_k$ represent the inverse of our charts $\varphi_k$, the decoders have to be injective. Moreover, since $\mathcal E_k=\mathcal T_k^{-1}\circ E_k$ represents the chart itself, the condition $\mathcal E_k\circ \mathcal D_k=\mathrm{Id}$ must be verified.
Therefore, we choose the encoder $E_k$ as a left-inverse of the corresponding decoder $D_k$.
More precisely, we use the decoders of the form
\begin{equation}\label{eq_generator}
D_k=T_L\circ A_L\circ\cdots \circ T_1\circ A_1,
\end{equation}
where the $T_l\colon\R^{d_l}\to\R^{d_l}$ are invertible neural networks and $A_l\colon\R^{d_{l-1}}\to\R^{d_l}$ are fixed injective linear operators  for $l=1,\dots,L$, $d=d_0<d_1<\cdots<d_L=n$. 
As it is a concatenation of injective mappings, we obtain that $D_k$ is injective.
Finally, the corresponding encoder is given by
\begin{equation}\label{eq_encoder}
E_k=A_1^\dagger \circ T_1^{-1}\circ\cdots\circ A_L^\dagger\circ T_L^{-1},\quad A^\dagger=(A^\tT A)^{-1}A^\tT.
\end{equation}
Then, it holds by construction that $\mathcal E_k\circ \mathcal D_k= \mathrm{Id}$.

In this paper, we build the invertible neural networks $T_l$ and the normalizing flows $\mathcal T_k$ out of coupling blocks as proposed by \citet{AKWR2018} based on the real NVP architecture \citep{DSB2016}.
To this end, we split the input $z\in\R^{d_l}$ into two parts $z=(z_1,z_2)\in\R^{d_{l}^1}\times\R^{d_{l}^2}$ 
with $d_l=d_{l}^1+d_{l}^2$ and define $T_l(z)=(x_1,x_2)$ with 
$$x_1=z_1 \, \mathrm{e}^{s_{2}(z_2)} + t_{2}(z_2)
\quad\text{and}\quad 
x_2=z_2 \, \mathrm{e}^{s_{1}(x_1)}+ t_{1}(x_1),
$$
where $s_{1},t_{1}\colon\R^{d_{l}^1}\to\R^{d_{l}^2}$ and $s_{2},t_{2}\colon\R^{d_{l}^2}\to\R^{d_{l}^1}$ are arbitrary subnetworks (depending on $l$).
Then, the inverse $T_l^{-1}(x_1,x_2)$ can analytically be derived as $z=(z_1,z_2)$ with 
$$
z_2=\big(x_2 - t_{1}(x_1)   \big) \,\mathrm{e}^{-s_{1}(x_1)}
\quad\text{and}\quad
z_1=\big(x_1 - t_{2}(z_2) \big) \,\mathrm{e}^{-s_{2}(z_2)}.
$$

\begin{remark}[Projection onto learned charts]\label{rem_projection}
Consider a decoder $D_k$ and an encoder $E_k$ as defined above.
By construction, the mapping $\pi_k=D_k\circ E_k$ is a (nonlinear) projection onto $\mathrm{range}(D_k)=\mathrm{range}(\pi_k)$, in the sense that $\pi_k\circ\pi_k=\pi_k$ and that $\pi_{k|\mathrm{range}(D_k)}$ is the identity on $\mathrm{range}(D_k)$.
Consequently, the mapping $\pi_k$ is a projection on the range of $D_k$ which represents the $k$-th chart of $\mathcal M$.
In particular, there is an open neighborhood $V\coloneqq\pi_k^{-1}(U_k)\subseteq\R^n$ such that $\pi_{k|V}$ is a projection onto $U_k$.
\end{remark}

\section{Optimization on Learned Manifolds}
\label{sec:opt_prob}

As motivated in the introduction, we are interested in  optimization problems of the form
\begin{equation}
\label{opt_prob}
\min_{x\in\R^n} F(x)\quad \text{subject to}\quad x\in\mathcal M,   
\end{equation}
where $F\colon\R^n\to\R$ is a differentiable function and $\mathcal M$ is available only through some data points.
In the previous section, we proposed a way to represent the manifold $\mathcal M$ by a mixture model $(D_k,E_k,\mathcal T_k)$ of VAEs.
This section outlines a gradient descent algorithm for the solution of \eqref{opt_prob} once the manifold is learned.

As outlined in the previous section, the inverse charts $\varphi_k^{-1}$ of the manifold $\mathcal M$ are modeled by $\mathcal D_k\coloneqq D_k\circ\mathcal T_k$. The chart $\varphi_k$ itself is given by the mapping $\mathcal E_k\coloneqq \mathcal T_k^{-1}\circ E_k$ restricted to the manifold.
For the special case of a VAE with a single generator $\mathcal D$, \citet{ASS2022,CKFB2020,DCE2021} propose to solve \eqref{opt_prob} in the latent space.
More precisely, starting with a latent initialization $z_0\in\R^d$ they propose to solve
$$
\min_{z\in\R^d} F(\mathcal D(z))
$$
using a gradient descent scheme.
However, when using multiple charts, such a gradient descent scheme heavily depends on the current chart.
Indeed, the following example shows that the gradient direction can change significantly, if we use a different chart.

\begin{example}
Consider the two-dimensional manifold $\R^2$ and the two learned charts given by the generators 
$$
\mathcal D_1(z_1,z_2)=(10 z_1,z_2),\quad\text{and}\quad\mathcal D_2(z_1,z_2)=(z_1,10 z_2).
$$
Moreover let $F\colon\R^2\to\R$ be given by $(x,y)\mapsto x+y$.
Now, the point $x^{(0)}=(0,0)$ corresponds for  both charts  to $z^{(0)}=(0,0)$.
A gradient descent step with respect to $F\circ \mathcal D_k$, $k=1,2$, using step size $\tau$ yields the 
latent values
\begin{align}
(z_1^{(1)},z_2^{(1)})&=z^{(0)}-\tau\nabla (F\circ \mathcal D_1)(z^{(0)})=-(10 \tau,\tau),\\
(\tilde z_1^{(1)},\tilde z_2^{(1)})&=z^{(0)}-\tau\nabla (F\circ \mathcal D_2)(z^{(0)})=-(\tau,10 \tau).
\end{align}
Thus, one gradient steps with respect to $F\circ\mathcal D_k$ yields the values
$$
x^{(1)}=\mathcal D_1(z^{(1)})=-(100 \tau,\tau),\quad \tilde x^{(1)}=\mathcal D_2(\tilde z^{(1)})=-(\tau,100\tau).
$$
Consequently, the gradient descent steps with respect to two different charts can point into completely different directions, independently of the step size $\tau$.
\end{example}

Therefore, we aim to use a gradient formulation which is independent of the parameterization of the manifold.
Here, we use the concept of the Riemannian gradient with respect to the Riemannian metric, which is inherited 
from the Euclidean space in which the manifold $\mathcal M$ is embedded.
To this end, we first revisit some basic facts about Riemannian gradients on embedded manifolds which can be found, e.g., in the book of \citet{AMS2009}.
Afterwards, we consider suitable retractions in order to perform a descent step in the direction of the negative Riemannian gradient.
Finally, we use these notions in order to derive a gradient descent procedure on a manifold given by mixtures of VAEs.

\subsection{Background on Riemannian Optimization}

\paragraph{Riemannian Gradients on Embedded Manifolds.}

Let $x\in\mathcal M\subseteq\R^n$ be a point on the manifold,  let $\varphi\colon U\to\R^d$ be a chart with $x\in U$ and $\varphi^{-1}\colon\varphi(U)\to U$ be its inverse.
Then, the tangent space is given by the set of all directions $\dot\gamma(0)$ of differentiable curves $\gamma\colon(-\epsilon,\epsilon)\to\mathcal M$ with $\gamma(0)=x$.
More precisely, it is given by the linear subspace of $\R^n$ defined as
\begin{equation}\label{eq:tangent_space}
T_x\mathcal M=\{Jy:y\in\R^d\}, \quad \text{ where } J\coloneqq \nabla \varphi^{-1}(\varphi(x))\in\R^{n\times d}.
\end{equation}
The tangent space inherits the Riemannian metric from $\R^n$, i.e., we equip the tangent space with the inner product
$$
\langle u,v\rangle_x=u^\tT v,\quad u,v\in T_x\mathcal M.
$$

A function $f\colon\mathcal M\to\R^m$ is called differentiable if for any differentiable curve $\gamma\colon(-\epsilon,\epsilon)\to\mathcal M$ we have that $f\circ \gamma\colon(-\epsilon,\epsilon)\to\R^m$ is differentiable.
In this case the differential of $f$ is defined by
$$
D f(x)\colon T_x\mathcal M\to\R^m,\qquad D f(x)[h]=\frac{\dx}{\dx t} f(\gamma_h(t))\Big|_{t=0},
$$
where $\gamma_h\colon(-\epsilon,\epsilon)\to\mathcal M$ is a curve with $\gamma_h(0)=x$ and $\dot\gamma_h(0)=h$.
Finally, the Riemannian gradient of a differentiable function $f\colon\mathcal M\to\R$ is given by the unique element $\nabla_{\mathcal M}f\in T_x\mathcal M$ which fulfills
$$
D f(x)[h]=\langle \nabla_{\mathcal M}f,h\rangle_x \quad\text{for all}\quad h\in T_x\mathcal M.
$$
\begin{remark}\label{rem:riem_grad_ex}
In the case that $f$ can be extended to a differentiable function on a neighborhood of $\mathcal M$, these formulas can be simplified.
More precisely, we have that the differential is given by
$
D f(x)[h]=h^\tT \nabla f(x)
$, where $\nabla f$ is the Euclidean gradient of $f$. In other words, $D f(x)$ is the Fréchet derivative of $f$ at $x$ restricted to $T_x\mathcal M$.
Moreover, the Riemannian gradient coincides with the orthogonal projection of $\nabla f$ on the tangent space, i.e.,
$$
\nabla_{\mathcal M}f(x)=P_{T_x\mathcal M}(\nabla f(x)),\quad P_{T_x\mathcal M}(y)=\argmin_{z\in T_x\mathcal M}\|y-z\|^2.
$$
Here the orthogonal projection can be rewritten as
$P_{T_x\mathcal M}=J (J^\tT J)^{-1} J^\tT$, $J=\nabla \varphi^{-1}(\varphi(x))$
such that the Riemannian gradient is given by $\nabla_{\mathcal M}f(x)=J (J^\tT J)^{-1} J^\tT\nabla f(x)$.
\end{remark}

\paragraph{Retractions.}

Once the Riemannian gradient is computed, we aim to perform a descent step in the direction 
of $-\nabla_{\mathcal M}f(x)$ on $\mathcal M$.
To this end, we need the concept of  retraction.
Roughly speaking, a retraction in $x$ maps a tangent vector $\xi$ to the point that is reached by moving
from $x$ in the direction $\xi$.
Formally, it is defined as follows.
\begin{definition}
A differentiable mapping $R_x\colon V_x\to\mathcal M$ for some neighborhood $V_x\subseteq T_x\mathcal M$ of $0$ is called a retraction in $x$, if $R_x(0)=x$ and
$$
D R_x(0)[h]=h\quad\text{for all}\quad h\in T_0(V_x)= T_x\mathcal M,
$$
where we identified $T_0(V_x)$ with $T_x\mathcal M$.
Moreover, a differentiable mapping $R=(R_x)_{x\in\mathcal M}\colon V\to\mathcal M$ on a subset of the tangent bundle $V=(V_x)_{x\in\mathcal M}\subseteq T\mathcal M=(T_x\mathcal M)_{x\in\mathcal M}$ is a retraction on $\mathcal M$, if for all $x\in \mathcal M$ we have that $R_x$ is a retraction in $x$.
\end{definition}

Now, let $R\colon V\to\mathcal M$ be a retraction on $\mathcal M$. 
Then, the Riemannian gradient descent scheme starting at $x_0\in\mathcal M$ with step size $\tau>0$ is defined by
$$
x_{t+1}=R_{x_t}(-\tau\nabla_{\mathcal M} f(x_t)).
$$

\subsection{Retractions for Learned Charts}

In order to apply this gradient scheme for a learned manifold given by the learned mappings $(\mathcal D_k,\mathcal E_k)_{k=1}^K$, 
we consider two types of retractions. 
We introduce them and show that they are indeed retractions in the following lemmas.
The first one generalizes the idea from Lemma 4 and Proposition 5 in \citet{AM2012} of moving along the 
tangent vector in $\R^n$ and reprojecting onto the manifold. 
However, the results by \citet{AM2012} are based on the orthogonal projection, which is hard or even impossible to compute.
Thus, we replace it by some more general projection $\pi$.
In our applications,  $\pi$ will be chosen as in Remark~\ref{rem_projection}, i.e., 
we set
$
\pi(x)=\mathcal D_k(\mathcal E_k(x)).
$

\begin{lemma}\label{lem:retr1}
Let $x\in\mathcal M$, $U_x\subseteq\R^n$ be a neighborhood of $x$ in $\R^n$, $\pi\colon U_x\to\mathcal M\cap U_x$ be a differentiable map such that  $\pi\circ\pi=\pi$. Set $V_x=\{h\in T_x\mathcal M\subseteq\R^n: x+h\in U_x\}$. Then
$$
R_x(h)=\pi(x+h),\qquad h\in V_x,
$$
defines a retraction in $x$.
\end{lemma}
\begin{proof}
The property $R_x(0)=x$ is directly clear from the definition of $R_x$.
Now let $h\in T_x\mathcal M\subseteq\R^n$ and $\gamma_h\colon(-\epsilon,\epsilon)\to\mathcal M$ be a differentiable curve with $\gamma_h(0)=x$ and $\dot\gamma_h(0)=h$.
As ${\pi}_{|U}$ is the identity on $\mathcal M$, we have by the chain rule that
$$
h=\dot\gamma_h(t) =\frac{\dx}{\dx t} \pi(\gamma_h(t)) \Big|_{t=0}=\nabla \pi(x)\dot\gamma_h(0)=\nabla \pi(x) h,
$$
where $\nabla \pi(x)$ is the Euclidean Jacobian matrix of $\pi$ at $x$.
Similarly,
\begin{equation*}
DR_x(0)[h]=\frac{\dx}{\dx t}R_x(th)\Big|_{t=0}=\frac{\dx}{\dx t}\pi(x+th)\Big|_{t=0}=\nabla \pi(x) h=h.    
\end{equation*}
\end{proof}

The second retraction uses the idea of changing to local coordinates, moving into the gradient direction by using the local coordinates and then going back to the manifold representation. 
Note that similar constructions are considered in Section 4.1.3 in the book of \citet{AMS2009}.
However, as we did not find an explicit proof for the lemma, we give it below for the sake of completeness.

\begin{lemma}\label{lem:retr2}
Let $x\in\mathcal M$ and denote by $\varphi\colon U\to\R^d$ a chart with $x\in U\subseteq \mathcal M$.
Then,
$$
R_x(h)=\varphi^{-1}(\varphi(x)+(J^\tT J)^{-1}J^\tT h), \quad J=\nabla \varphi^{-1}(\varphi(x))
$$
defines a retraction in $x$.
\end{lemma}
\begin{proof}
The property $R_x(0)=x$ is directly clear from the definition of $R_x$.
Now let $h\in T_0(T_x\mathcal M)=T_x\mathcal M\subseteq\R^n$.
By \eqref{eq:tangent_space}, we have that there exists some $y\in\R^d$ such that $h=Jy$.
Then, we have by the chain rule that
$$
DR_x(0)[h]=\frac{\dx}{\dx t}R_x(th)\Big|_{t=0}=(\nabla \varphi^{-1}(\varphi(x))) (J^\tT J)^{-1} J^\tT h=J(J^\tT J)^{-1} J^\tT Jy=Jy=h.
$$
\end{proof}

\subsection{Gradient Descent on Learned Manifolds}

By Lemma~\ref{lem:retr1} and \ref{lem:retr2}, we obtain that the mappings
\begin{equation}\label{eq:retractions}
R_{k,x}(h)=\mathcal D_k(\mathcal E_k(x+h))\quad\text{and}\quad \tilde R_{k,x}(h)=\mathcal D_k(\mathcal E_k(x)+(J^\tT J)^{-1}J^\tT h)
\end{equation}
with $J=\nabla \mathcal D_k(\mathcal E_k(x))$
are retractions in all $x\in U_k$.
If we define $R$ such that $R_x$ is given by $R_k$ for some $k$ such that $x\in U_k$,
then the differentiability of $R=(R_x)_{x\in\mathcal M}$ in $x$ might be violated.
Moreover, the charts learned by a mixture of VAEs only overlap approximately and not exactly.
Therefore, these retractions cannot be extended to a retraction on the whole manifold $\mathcal M$ in general.
As a remedy, we propose the following gradient descent step on a learned manifold.

Starting from a point $x_n$, we first compute for $k=1,\dots,K$ the probability, that $x_n$ belongs to the $k$-th chart. By \eqref{eq_true_betas}, this probability can be approximated by
\begin{equation}\label{eq:beta_descent}
\beta_k\coloneqq \tfrac{\alpha_k\exp(\mathrm{ELBO}(x_n|\theta_k))}{\sum_{j=1}^K \alpha_j\exp(\mathrm{ELBO}(x_n|\theta_j))}.
\end{equation}
Afterwards, we project $x_n$ onto the $k$-th chart by applying $\tilde x_{n,k}=\mathcal D_k(\mathcal E_k(x_n))$ (see Remark~\ref{rem_projection}) and compute the Riemannian gradient $g_{n,k}=\nabla_{\mathcal M}F(\tilde x_{n_k})$.
Then, we apply the retraction $R_{k,\tilde x_{n,k}}$ (or $\tilde R_{k,\tilde x_{n,k}}$) to perform a gradient descent step  
$x_{n+1,k}=R_{k,\tilde x_{n,k}}(-\tau_n g_{n,k})$.
Finally, we average the results by $x_{n+1}=\sum_{k=1}^K\beta_{k}x_{n+1,k}$.

The whole gradient descent step is summarized in Algorithm~\ref{alg:gd_manifold}.
Finally, we  compute the sequence $(x_n)_n$ by applying Algorithm~\ref{alg:gd_manifold} iteratively.

\begin{algorithm}[t]
\begin{algorithmic}
\State Inputs: Function $F\colon\mathcal M\to\R$, point $x_n$, step size $\tau_n>0$.
\For{$k=1,\dots,K$}
\State - Approximate the probability that $x_{n}$ belongs to chart $k$ by computing the
$
\beta_{k}
$
\State \phantom{- }from \eqref{eq:beta_descent}.
\State - Project to the $k$-th chart by $\tilde x_{n,k}=\mathcal D_k(\mathcal E_k(x_n))$.
\State - Compute the Riemannian gradient $g_{n,k}=\nabla_{\mathcal M} F(\tilde x_{n,k})$, e.g., by Remark~\ref{rem:riem_grad_ex}.
\State - Perform a gradient descent with the retraction $R_{k,\tilde x_{n,k}}$, i.e., define
\State \phantom{- }$
x_{n+1,k}=R_{k,\tilde x_{n,k}}(-\tau_n g_{n,k}).$
\EndFor
\State - Average results by computing
$
x_{n+1}=\sum_{k=1}^K\beta_k x_{n+1,k}.
$
\end{algorithmic}
\caption{One gradient descent step on a learned manifold.}
\label{alg:gd_manifold}
\end{algorithm}

For some applications the evaluation of the derivative of $F$ is computationally costly.
Therefore, we aim to take  as large step sizes $\tau_n$ as possible in Algorithm~\ref{alg:gd_manifold}.
On the other hand, large step sizes can lead to numerical instabilities and divergence.
To this end, we use an adaptive step size selection as outlined in Algorithm~\ref{alg:adaptive_steps}.

\begin{remark}[Descent algorithm]
By construction, Algorithm~\ref{alg:adaptive_steps} is a descent algorithm. That is, for a sequence $(x_n)_{n}$ generated by the algorithm it holds that $F(x_{n+1})\leq F(x_n)$. 
With the additional assumption that $F$ is bounded from below, we have that $(F(x_{n}))_n$ is a bounded descending and hence convergent sequence.
However, this does neither imply convergence of the iterates $(x_n)_n$ themselves nor optimality of the limit of $(F(x_n))_n$.
For more details on the convergence of line-search algorithms on manifolds we refer to Section 4 of the book by \citet{AMS2009}.
\end{remark}

\begin{algorithm}[t]
\begin{algorithmic}
\State Input: Function $F$, initial point $x_0$, initial step size $\tau_0$.
\For{n=0,1,\dots}
\State Compute $x_{n+1}$ by Algorithm~\ref{alg:gd_manifold} with step size $\tau_n$.
\While{$F(x_{n+1})>F(x_n)$}
\State Update step size by $\tau_n\leftarrow \frac{\tau_n}2$.
\State Update $x_{n+1}$ by Algorithm~\ref{alg:gd_manifold} with the new step size $\tau_n$.
\EndWhile
\State Set step size for the next step $\tau_{n+1}=\frac{3\tau_n}2$.
\EndFor
\end{algorithmic}
\caption{Adaptive step size scheme for gradient descent on learned manifolds}
\label{alg:adaptive_steps}
\end{algorithm}

\section{Numerical Examples}\label{sec:toy_examples}

\begin{figure}
\centering
\begin{subfigure}[t]{0.19\textwidth}
\includegraphics[width=\textwidth]{imgs/data_two_circles}
\caption*{Two circles}
\end{subfigure}
\begin{subfigure}[t]{0.19\textwidth}
\includegraphics[width=\textwidth]{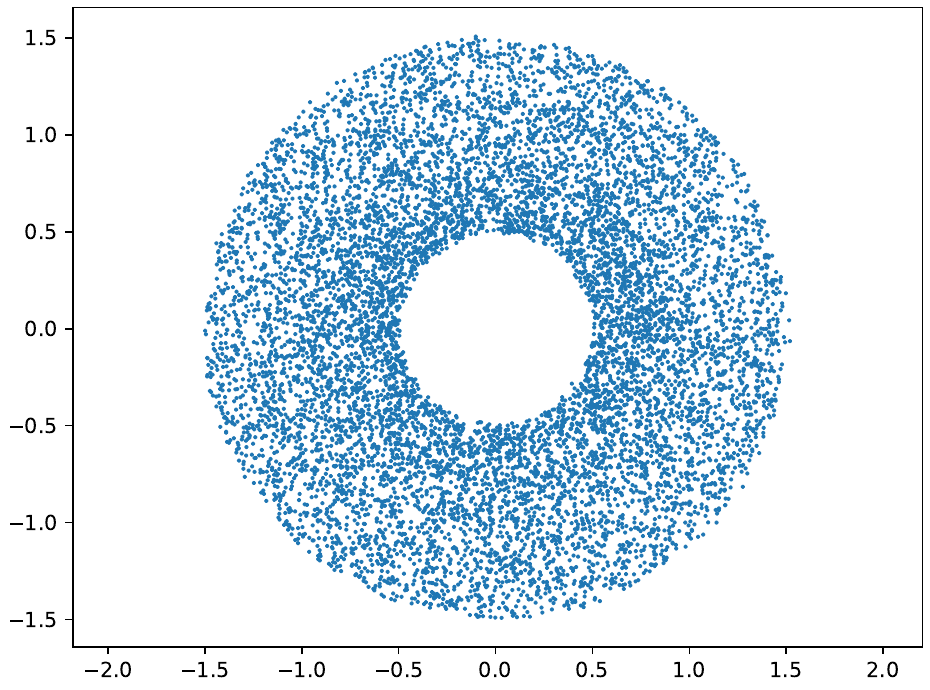}
\caption*{Ring}
\end{subfigure}
\begin{subfigure}[t]{0.19\textwidth}
\includegraphics[width=\textwidth]{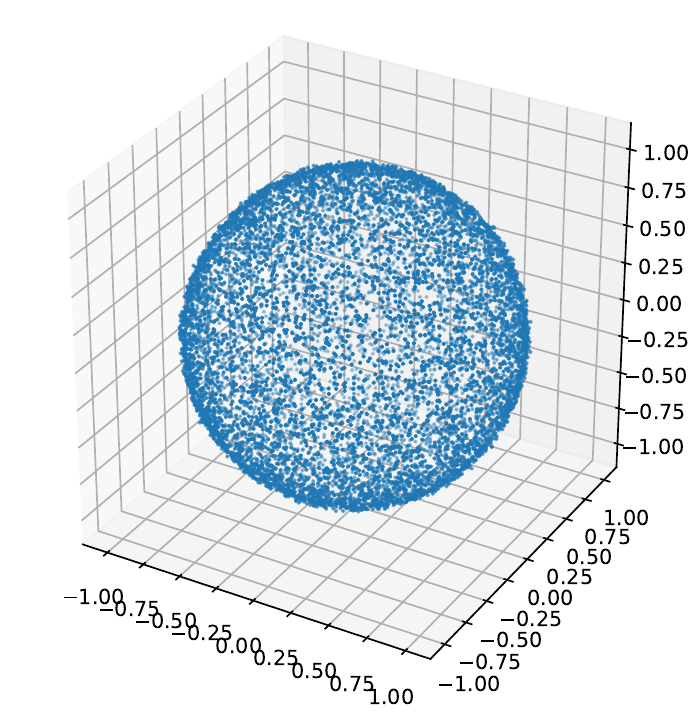}
\caption*{Sphere}
\end{subfigure}
\begin{subfigure}[t]{0.19\textwidth}
\includegraphics[width=\textwidth]{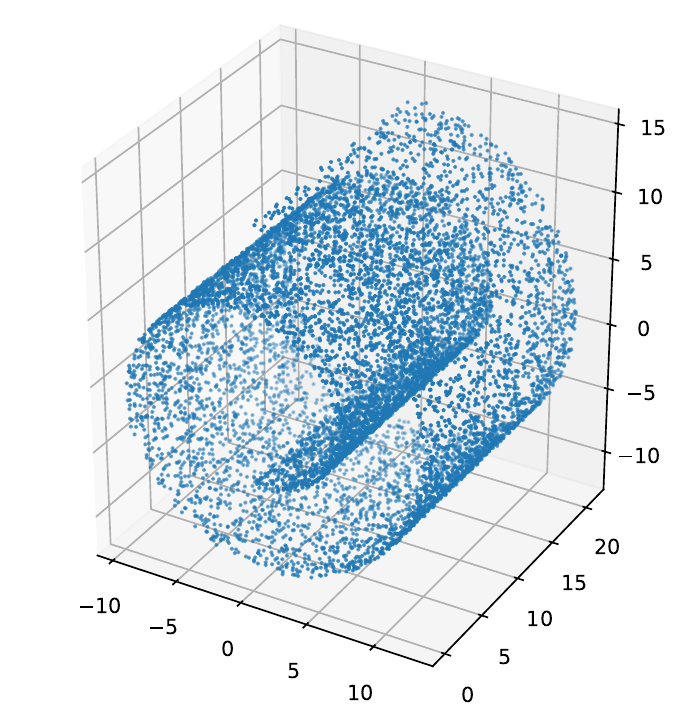}
\caption*{Swiss roll}
\end{subfigure}
\begin{subfigure}[t]{0.19\textwidth}
\includegraphics[width=\textwidth]{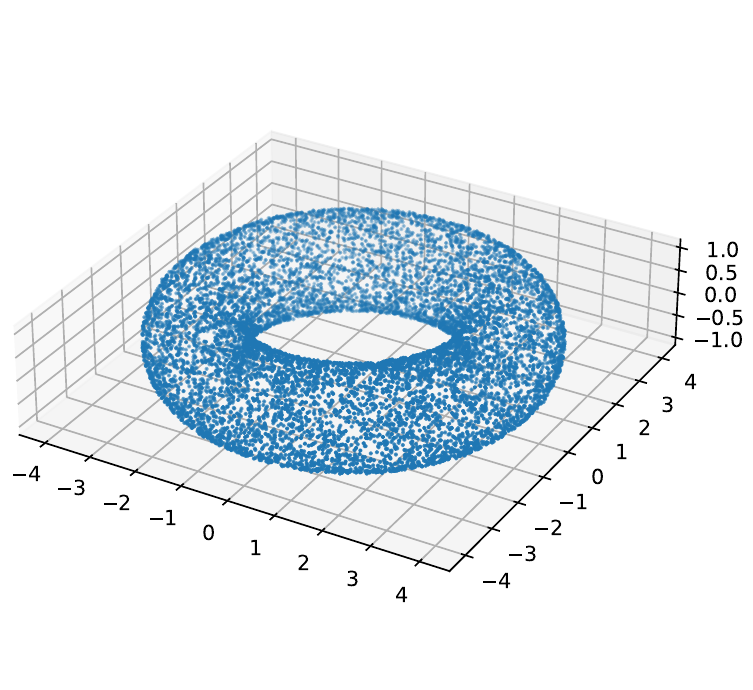}
\caption*{Torus}
\end{subfigure}
\caption{Data sets used for the different manifolds.}
\label{fig:manifolds_data_toy}
\end{figure}

Next, we test the numerical performance of the proposed method. 
In this section, we start with some one- and two-dimensional manifolds embedded in the two- or three-dimensional Euclidean space.
We use the architecture from Section~\ref{sec:architectures} with $L=1$.
That is, for all the manifolds, the decoder is given by $\mathcal T\circ A$ where $A\colon\R^d\to\R^n$ is given by 
$x\mapsto (x,0)$ if $d=n-1$ and by $A=\mathrm{Id}$ if $d=n$ and $\mathcal T$ is an invertible neural network with $5$ invertible coupling blocks where the subnetworks have two hidden layers and $64$ neurons in each layer. 
The normalizing flow modelling the latent space consists of $3$ invertible coupling blocks with the same architecture.
We train the mixture of VAEs for $200$ epochs with the Adam optimizer. Afterwards we apply the overlapping procedure for $50$ epochs, as in Algorithm~\ref{alg:training}.

We consider the manifolds ``two circles'', ``ring'', ``sphere'', ``swiss roll'' and ``torus''.
The (noisy) training data are visualized in Figure~\ref{fig:manifolds_data_toy}.
The number of charts $K$ is given in the following table.
\begin{center}
\begin{tabular}{c|ccccc}
&Two circles&Ring&Sphere&Swiss roll&Torus\\\hline
Number of charts&$4$&$2$&$2$&$4$&$6$
\end{tabular}
\end{center}
We visualize the learned charts in Figure~\ref{fig:learned_charts_toy}. Moreover, additional samples generated by the learned mixture of VAEs are shown in Figure~\ref{fig:generated_samples_toy}.
We observe that our model covers all considered manifolds and provides a reasonable approximation of different charts.
Finally, we test the gradient descent method from Algorithm~\ref{alg:gd_manifold} with some linear and quadratic functions, which often appear as data fidelity terms in inverse problems:
\begin{itemize}
\item $F(x)=x_2$ on the manifold ``two circles'' with initial points $\frac{x_0}{\|x_0\|}\pm(1.5,0)$ for $x_0=(\pm 0.2,1)$;
\item $F(x)=\|x-(-1,0)\|^2$ on the manifold ``ring'' with initial points $(1,\pm0.4)$;
\item $F(x)=\|x-(0,0,-2)\|^2$ on the manifold ``sphere'' with inital points $x_0/\|x_0\|$ for $x_0\in\{(0.3\cos(\tfrac{\pi k}{5}),0.3\sin(\tfrac{\pi k}{5}),1):k=0,\dots,9)\}$;
\item $F(x)=\|x-(-5,0,0)\|^2$ on the manifold ``torus'', where the inital points are drawn randomly from the training set.
\end{itemize}
We use the retraction from Lemma~\ref{lem:retr1} with a step size of $0.01$. 
The resulting trajectories are visualized in Figure~\ref{fig:trajectories_toy}.
We observe that all the trajectories behave as expected and approach the closest minimum of the objective function, even if this is not in the same chart of the initial point.

\begin{figure}
\centering
\begin{subfigure}[t]{0.24\textwidth}
\includegraphics[width=\textwidth]{imgs/plot_two_circles_4_gen.pdf}
\caption*{Two circles}
\end{subfigure}
\begin{subfigure}[t]{0.24\textwidth}
\includegraphics[width=\textwidth]{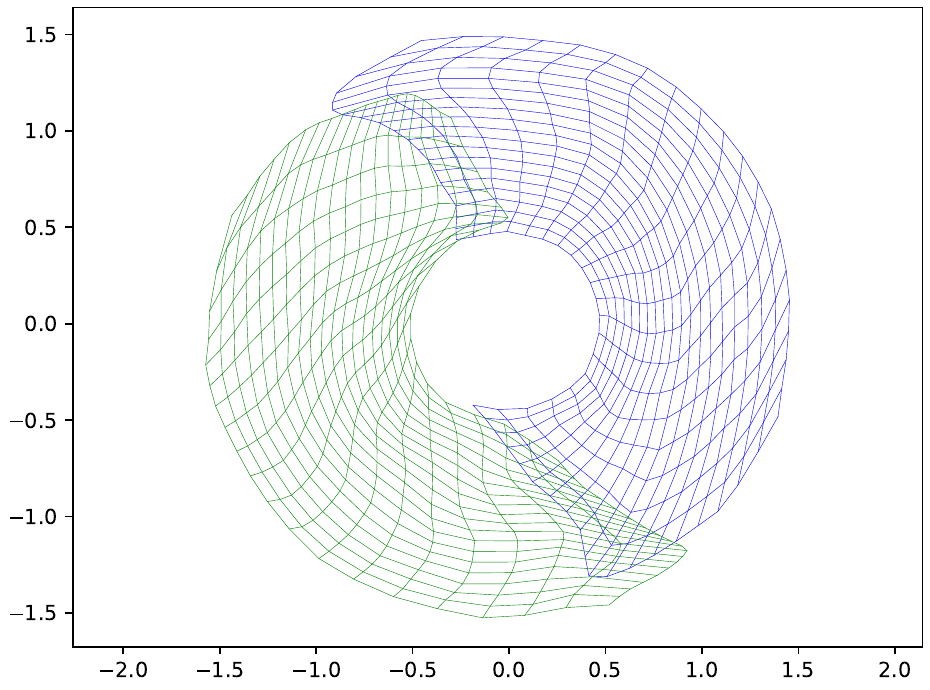}
\caption*{Ring}
\end{subfigure}
\begin{subfigure}[t]{0.48\textwidth}
\includegraphics[width=\textwidth]{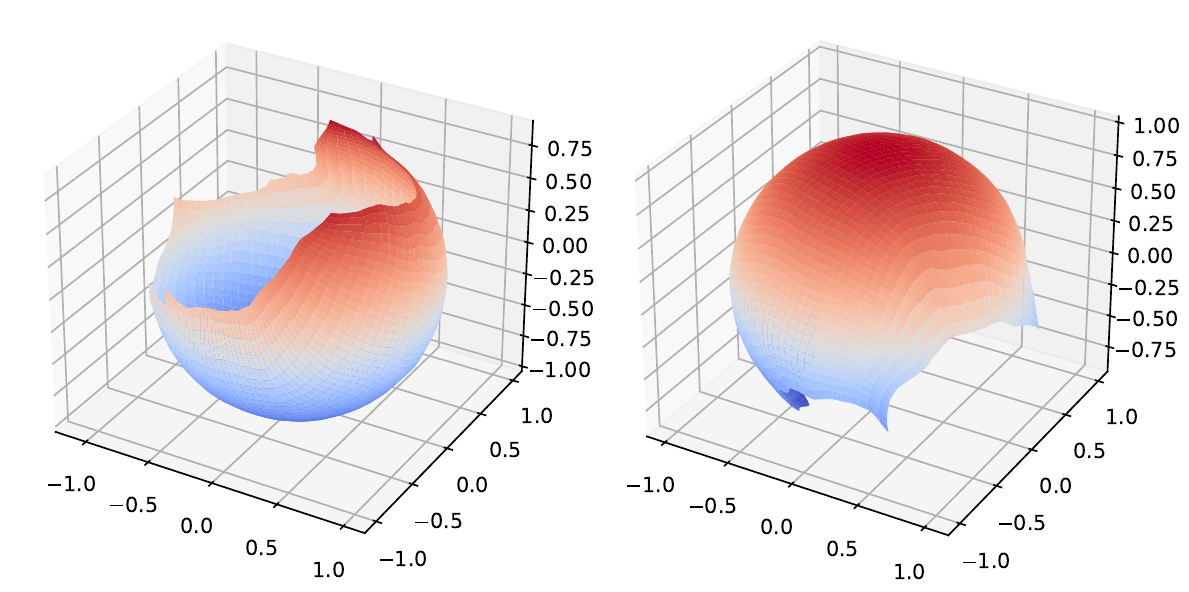}
\caption*{Sphere}
\end{subfigure}
\begin{subfigure}[t]{\textwidth}
\includegraphics[width=\textwidth]{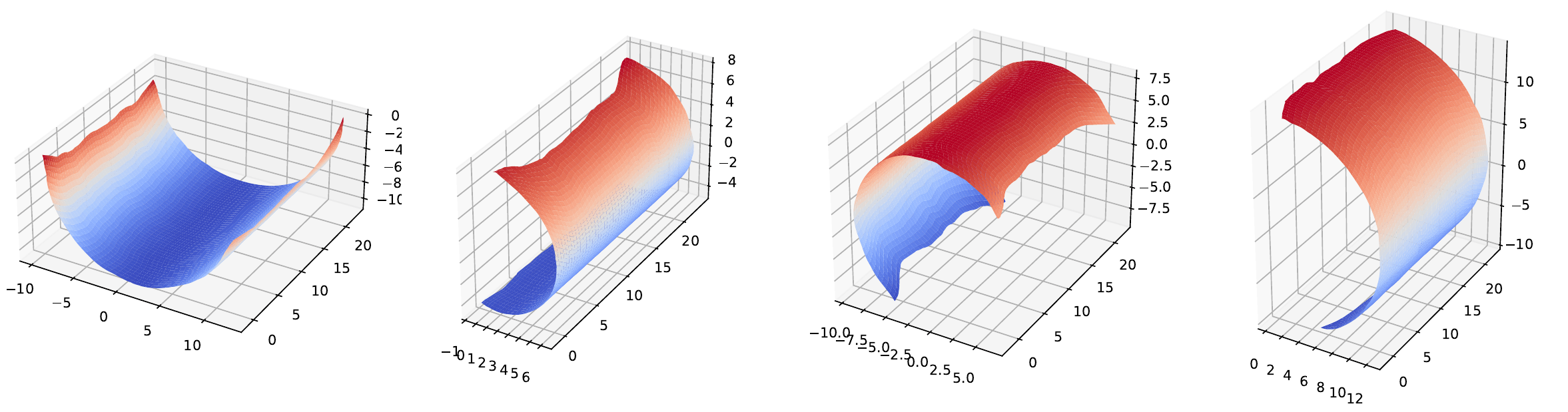}
\caption*{Swiss roll}
\end{subfigure}
\begin{subfigure}[t]{\textwidth}
\includegraphics[width=\textwidth]{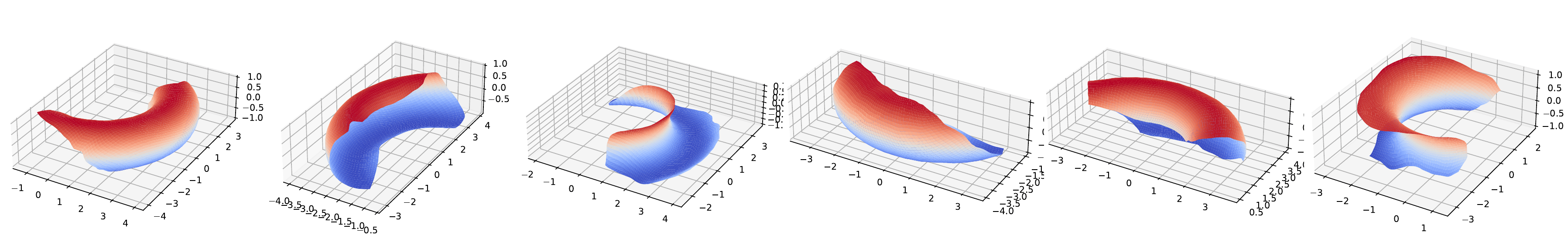}
\caption*{Torus}
\end{subfigure}
\caption{Learned charts for the different manifolds. For the manifolds ``two circles'' and ``ring'', each color represents one chart. For the manifolds ``sphere'', ``swiss roll'' and ``torus'' we plot each chart in a separate figure.}
\label{fig:learned_charts_toy}
\end{figure}

\begin{figure}
\centering
\begin{subfigure}[t]{0.19\textwidth}
\includegraphics[width=\textwidth]{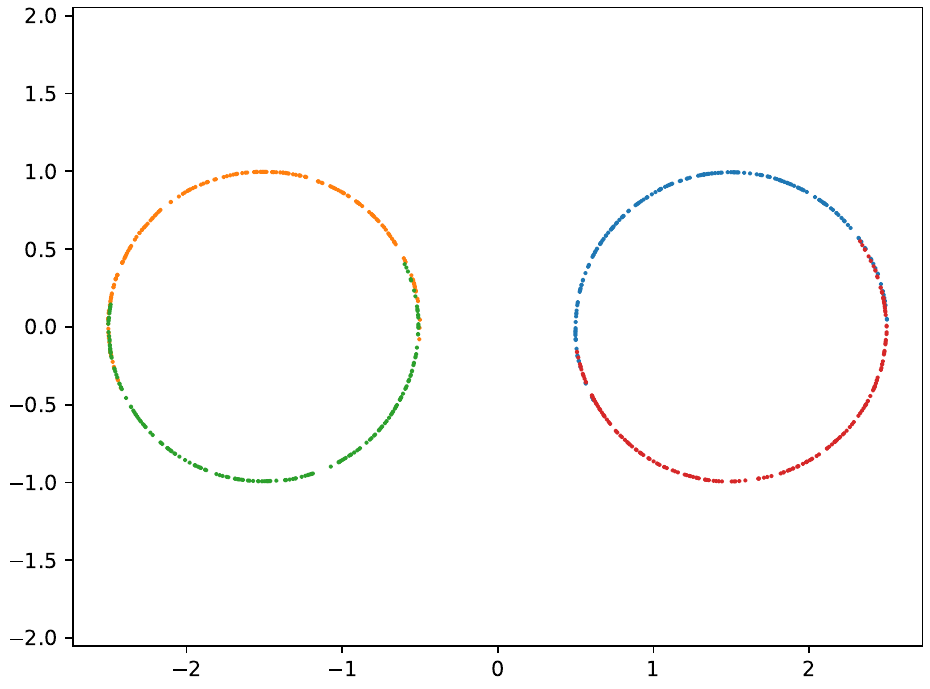}
\caption*{Two circles}
\end{subfigure}
\begin{subfigure}[t]{0.19\textwidth}
\includegraphics[width=\textwidth]{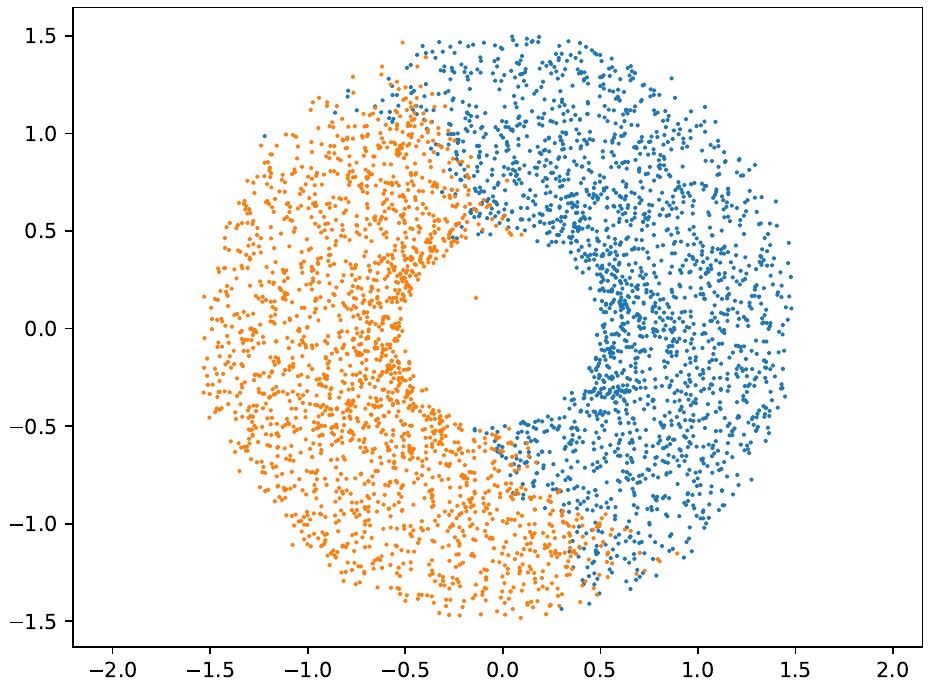}
\caption*{Ring}
\end{subfigure}
\begin{subfigure}[t]{0.19\textwidth}
\includegraphics[width=\textwidth]{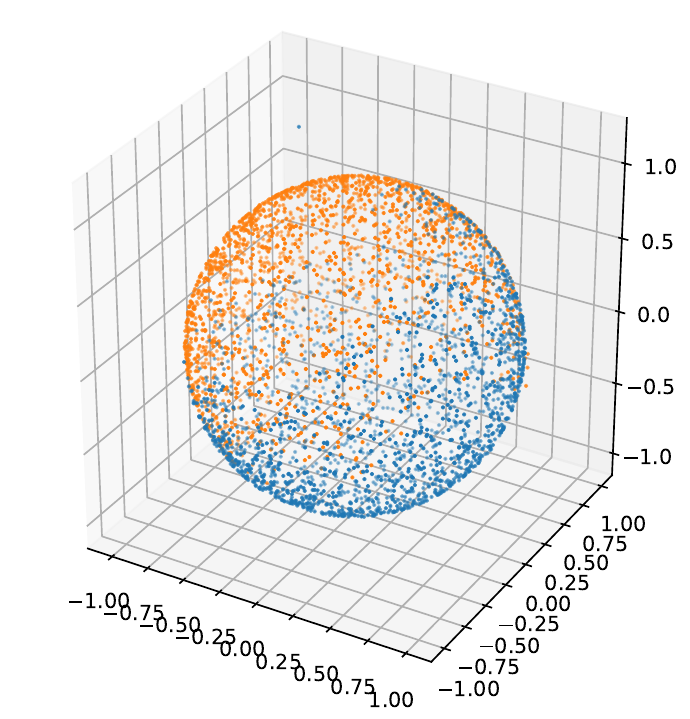}
\caption*{Sphere}
\end{subfigure}
\begin{subfigure}[t]{0.19\textwidth}
\includegraphics[width=\textwidth]{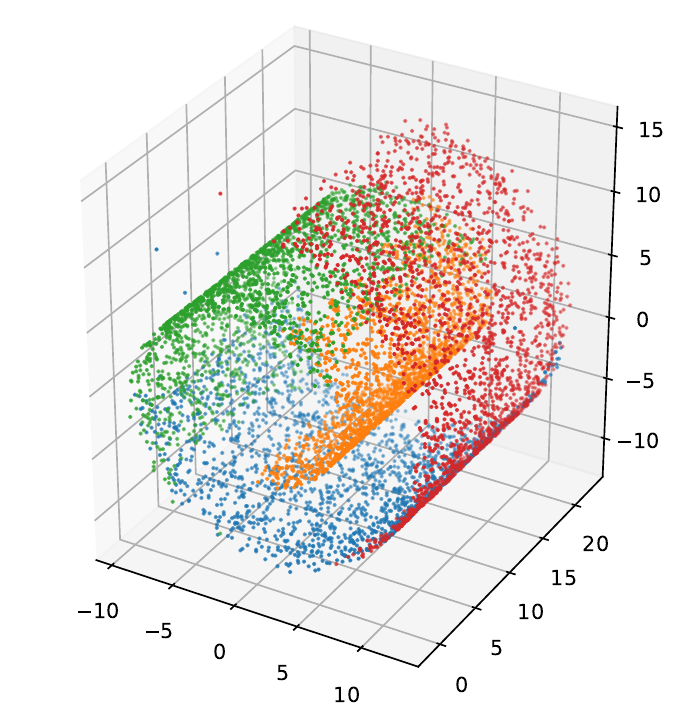}
\caption*{Swiss roll}
\end{subfigure}
\begin{subfigure}[t]{0.19\textwidth}
\includegraphics[width=\textwidth]{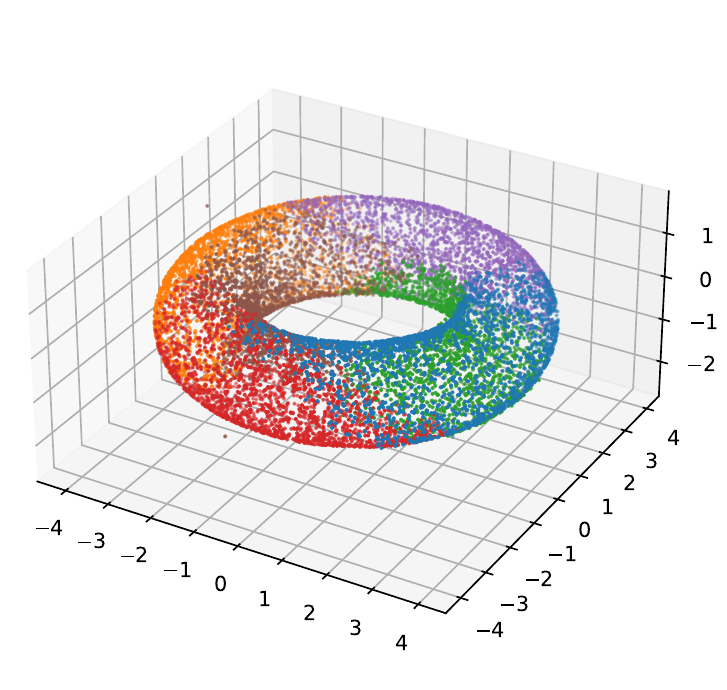}
\caption*{Torus}
\end{subfigure}
\caption{Generated samples by the learned mixture of VAEs. The color of a point indicates from which generator the point was sampled.}
\label{fig:generated_samples_toy}
\end{figure}

\begin{figure}
\centering
\begin{subfigure}[t]{0.24\textwidth}
\includegraphics[width=\textwidth]{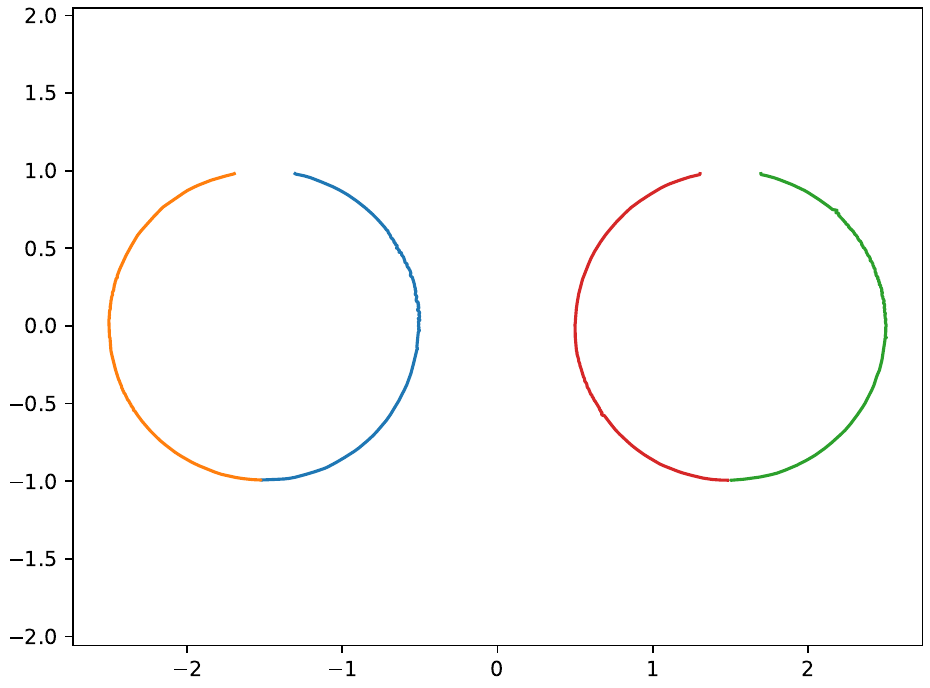}
\caption*{Two circles}
\end{subfigure}
\begin{subfigure}[t]{0.24\textwidth}
\includegraphics[width=\textwidth]{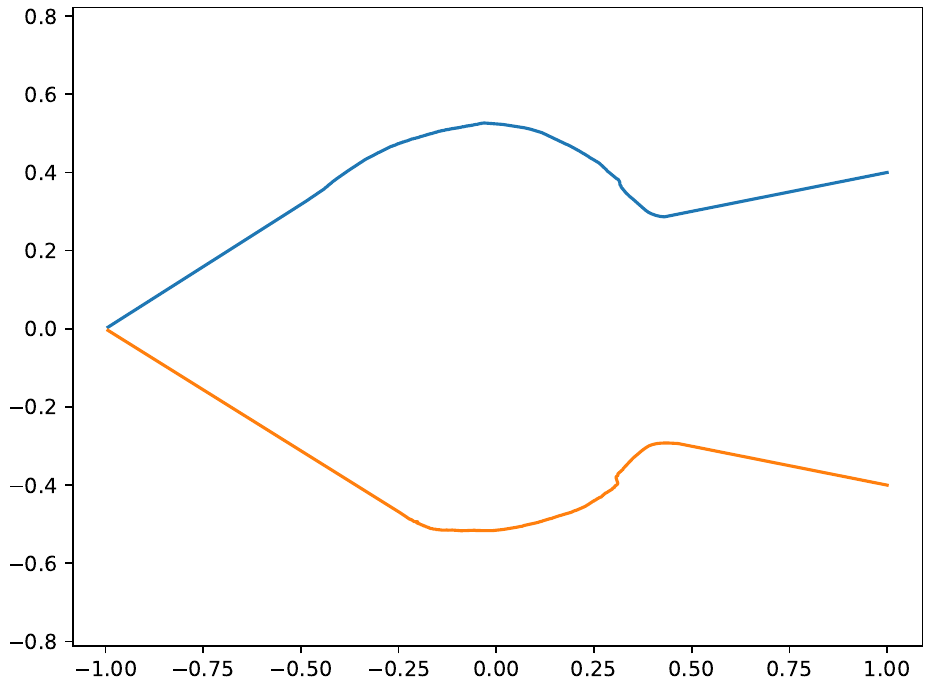}
\caption*{Ring}
\end{subfigure}
\begin{subfigure}[t]{0.24\textwidth}
\includegraphics[width=\textwidth]{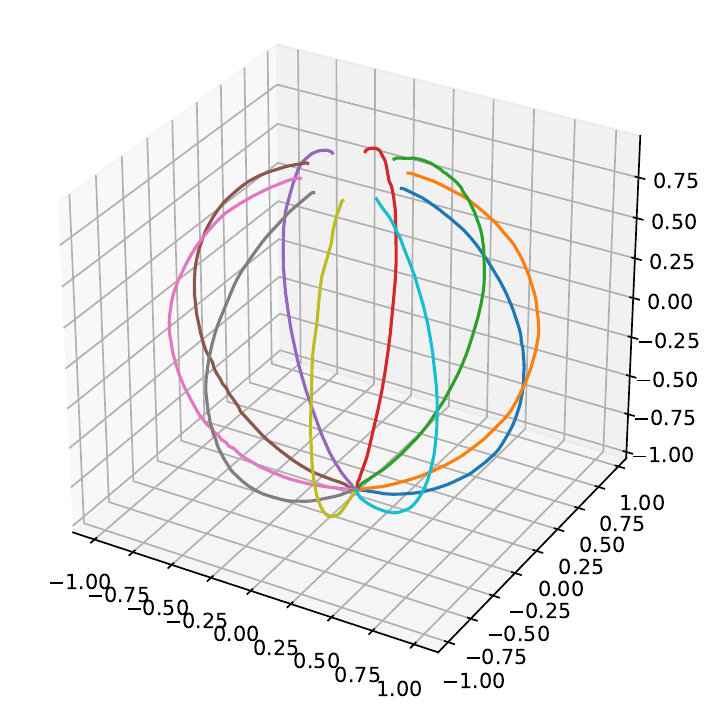}
\caption*{Sphere}
\end{subfigure}
\begin{subfigure}[t]{0.24\textwidth}
\includegraphics[width=\textwidth]{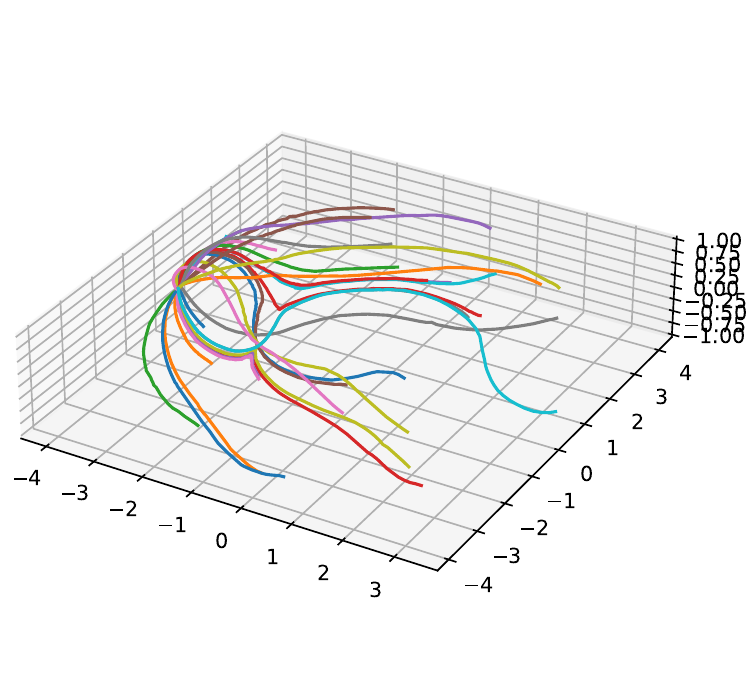}
\caption*{Torus}
\end{subfigure}
\caption{Trajectories of the gradient descent on the learned manifolds.}
\label{fig:trajectories_toy}
\end{figure}

\begin{remark}[Dimension of the Manifold]
For all our numerical experiments, we assume that the dimension $d$ of the data manifold is known. 
This assumption might be violated for practical applications.
However, there exist several methods in the literature to estimate the dimension of a manifold from data (see, e.g., \citealp{BSS2022,CS2016,FGQZ2010,LB2004}).
We are aware that dimension estimation of high dimensional data sets is a hard problem which cannot be considered as completely solved so far.
In particular, most of these algorithms make assumptions on the distribution of the given data points.
Even though it is an active area of research, it is not the scope of this paper to test, benchmark or develop such algorithms.
Similarly, combining them with our mixture of VAEs is left for future research.
\end{remark}

\section{Mixture of VAEs for Inverse Problems}\label{sec:inverse_problems}

In this section we describe how to use mixture of VAEs to solve inverse problems.
We consider an inverse problem of the form
\begin{equation}
\label{IP}
    y = \mathcal{G}(x)+\eta,
\end{equation}
where $\mathcal{G}$ is a possibly nonlinear map between $\R^n$ and $\R^m$, modelling a measurement (forward) operator, $x \in \R^n$ is a quantity to be recovered, $y \in \R^m$ is the noisy data and $\eta$ represents some noise. In particular, we analyze a linear and a nonlinear inverse problem: a deblurring problem and a parameter identification problem for an elliptic PDE arising in  electrical impedance tomography (EIT), respectively.

In many inverse problems, the unknown $x$ can be modeled as an element of a low-dimensional manifold $\mathcal M$ in $\R^n$ \citep{ASS2022,ASA2020,BAPD2017,HBLLS2021,OJMBDW2020,SKJLH2019,2022-massa-garbarino-benvenuto,2023-alberti-santacesaria}, and this manifold can be represented by the mixture of VAEs as explained in Section~\ref{sec:chart_learning}. Thus, the solution of \eqref{IP} can be found by optimizing the function 
\begin{equation}\label{eq:objective_function}
F(x) = \frac{1}{2} \| \mathcal{G}(x) - y \|^2_{\R^m}\quad\text{subject to }x\in\mathcal{M},
\end{equation}
by using the iterative scheme proposed in Section~\ref{sec:opt_prob}.

We would like to emphasize that the main goal of our experiments is not to obtain state-of-the-art results. 
Instead, we want to highlight the advantages of 
using multiple generators via a mixture of VAEs.
All our experiments are designed in such a way that the manifold property of the data is directly clear.
The application to real-world data and the combination with other methods in order to achieve competitive results are not within the scope
of this paper and are left to future research.

\paragraph{Architecture and Training.}
Throughout these experiments we consider images of size $128\times128$ and use the architecture from Section~\ref{sec:architectures} with $L=3$.
Starting with the latent dimension $d$, the mapping $A_1\colon\R^d\to\R^{32^2}$ fills up the input vector $x$ with zeros up to the size $32^2$, i.e., we set $A_1(x)=(x,0)$. The invertible neural network $T_1\colon\R^{32^2}\to\R^{32^2}$ consists of $3$ invertible blocks, where the subnetworks $s_{i}$ and $t_{i}$, $i=1,2$ are dense feed-forward networks with two hidden layers and $64$ neurons.
Afterwards, we reorder the dimensions to obtain an image of size $32\times32$.
The mappings $A_2\colon\R^{32\times 32}\to\R^{32\times32\times4}$ and $A_3\colon\R^{64\times64}\to\R^{64\times64\times 4}$ copy each channel $4$ times.
Then, the generalized inverses $A_2^\dagger\colon\R^{32\times32\times4}\to\R^{32\times32}$ and $A_3^\dagger\colon\R^{64\times64\times4}\to\R^{64\times64}$ from \eqref{eq_encoder} are given by taking the mean of the four channels of the input.
The invertible neural networks $T_{2}\colon\R^{32\times32\times4}\to\R^{64\times 64}$ and $T_3\colon\R^{64\times64\times 4}\to\R^{128\times128}$ consist of $3$ invertible blocks, where the subnetworks $s_i$ and $t_i$, $i=1,2$ are convolutional neural networks with one hidden layer and $64$ channels.
After these three coupling blocks we use an invertible upsampling \citep{EKS2020} to obtain the correct output dimension.
For the normalizing flow in the latent space, we use an invertible neural network with three blocks, where the subnetworks $s_{i}$ and $t_{i}$, $i=1,2$ are dense feed-forward networks with two hidden layers and $64$ neurons.

We train all the models for 200 epochs with the Adam optimizer. Afterwards we apply the overlapping procedure for $50$ epochs. See Algorithm~\ref{alg:training} for the details  of the training algorithm.

\subsection{Deblurring}

\begin{figure}[t]
\begin{subfigure}[t]{\textwidth}
\includegraphics[width=\textwidth]{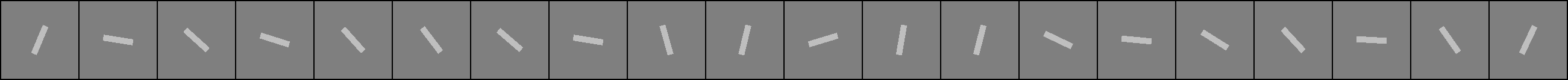}
\caption{Samples from the considered data set for the deblurring example.}
\label{fig:dataset_deblurring}
\end{subfigure}

\begin{subfigure}[t]{\textwidth}
\includegraphics[width=\textwidth]{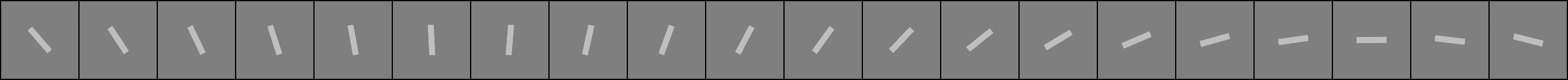}
\caption{Learned chart with one generator. The figure shows the images $D(x)$ for $20$ values of $x$ equispaced in $[-1,1]$.}
\label{fig:bar_chart_1_gen}
\end{subfigure}

\begin{subfigure}[t]{\textwidth}
\includegraphics[width=\textwidth]{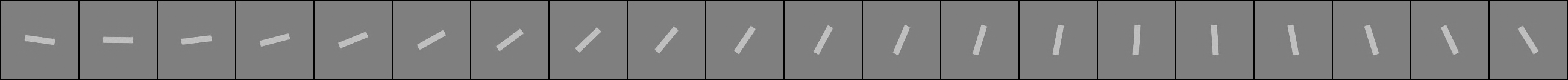}
\includegraphics[width=\textwidth]{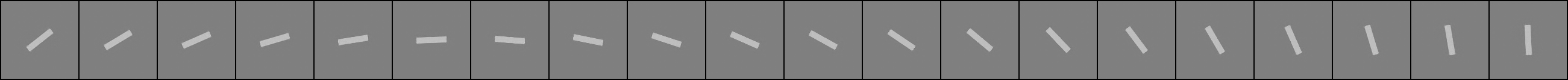}
\caption{Learned charts with two generators. The figure shows the images $D_k(x)$ for $20$ values of $x$ equispaced in $[-1,1]$ for $k=1$ (top) and $k=2$ (bottom).}
\label{fig:bar_chart_2_gen}
\end{subfigure}
\begin{subfigure}[t]{\textwidth}
\includegraphics[width=\textwidth]{imgs/bar_ds_grid}
\includegraphics[width=\textwidth]{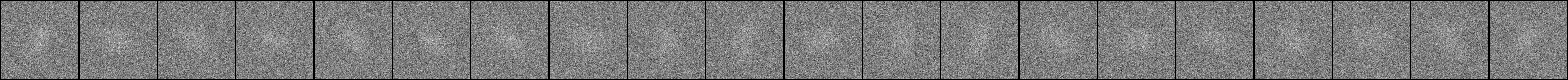}
\includegraphics[width=\textwidth]{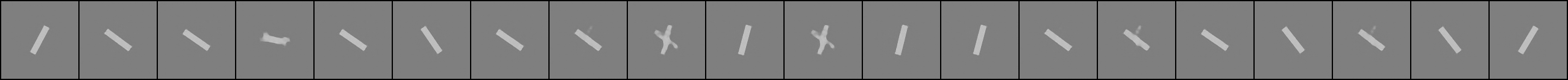}
\includegraphics[width=\textwidth]{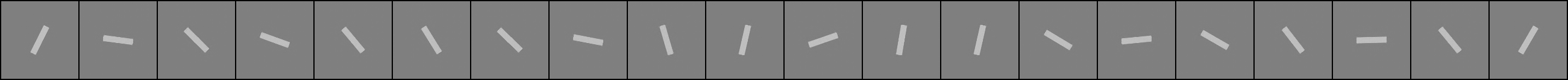}
\caption{Reconstructions for the deblurring example. From top to bottom: ground truth image, observation, reconstruction with one generator and reconstruction with two generators.}
\label{fig:deblurring_reconstructions}
\end{subfigure}
\caption{Data set, learned charts and reconstructions for the deblurring example.}
\end{figure}

\begin{figure}
\begin{subfigure}[t]{\textwidth}
\centering
\includegraphics[width=.05\textwidth]{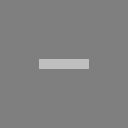}
\includegraphics[width=.05\textwidth]{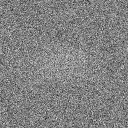}
\caption{Ground truth (left) and observation (right).}
\label{fig:traj_gt_obs}
\end{subfigure}
\begin{subfigure}[t]{\textwidth}
\includegraphics[width=\textwidth]{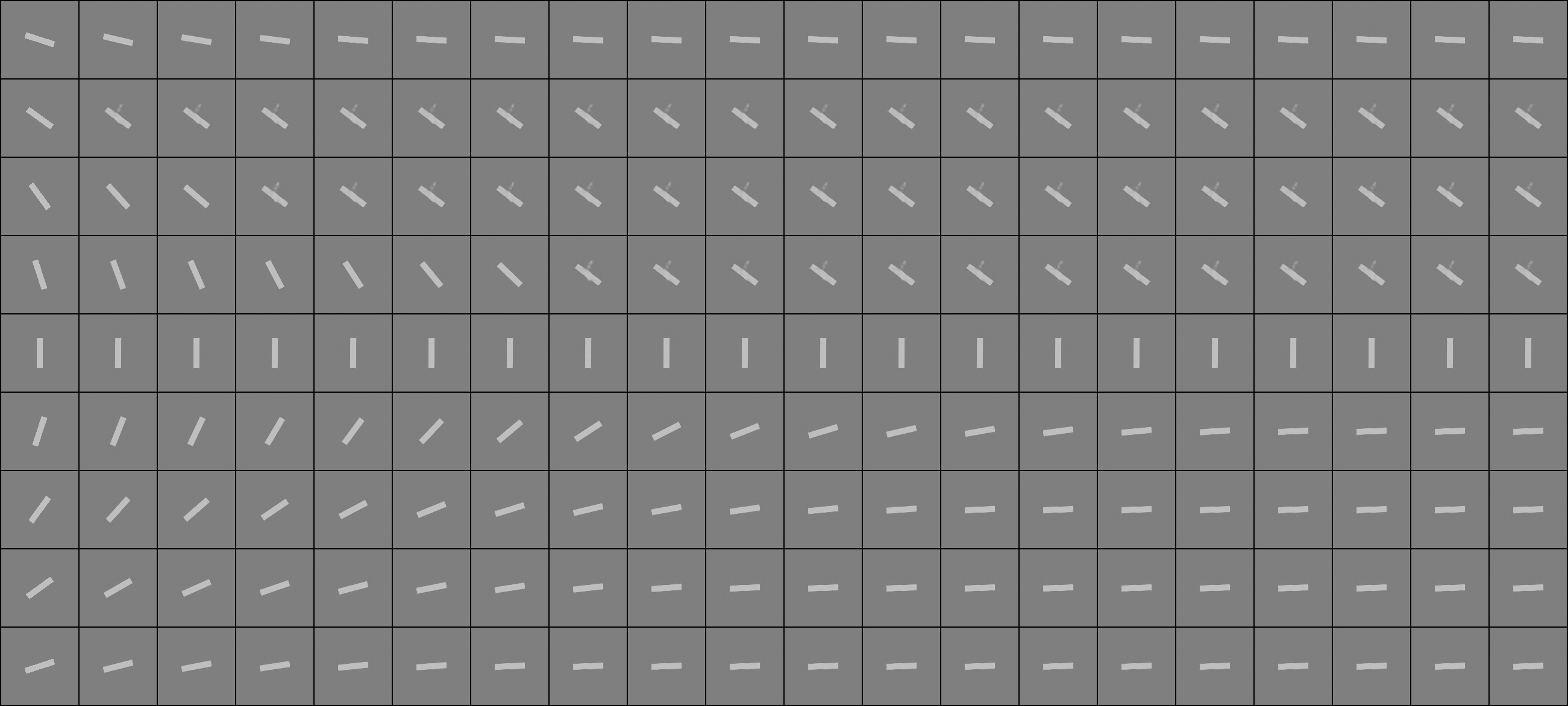}
\caption{Visualization of the trajectories $(x_n)_n$ for different initializations $x_0$ with one generator.
Left column: initialization, right column: reconstruction $x_{250}$, columns in between: images $x_n$ for $n$ approximately equispaced between $0$ and $250$.}
\label{fig:deblurring_one_chart}
\end{subfigure}
\begin{subfigure}[t]{\textwidth}
\includegraphics[width=\textwidth]{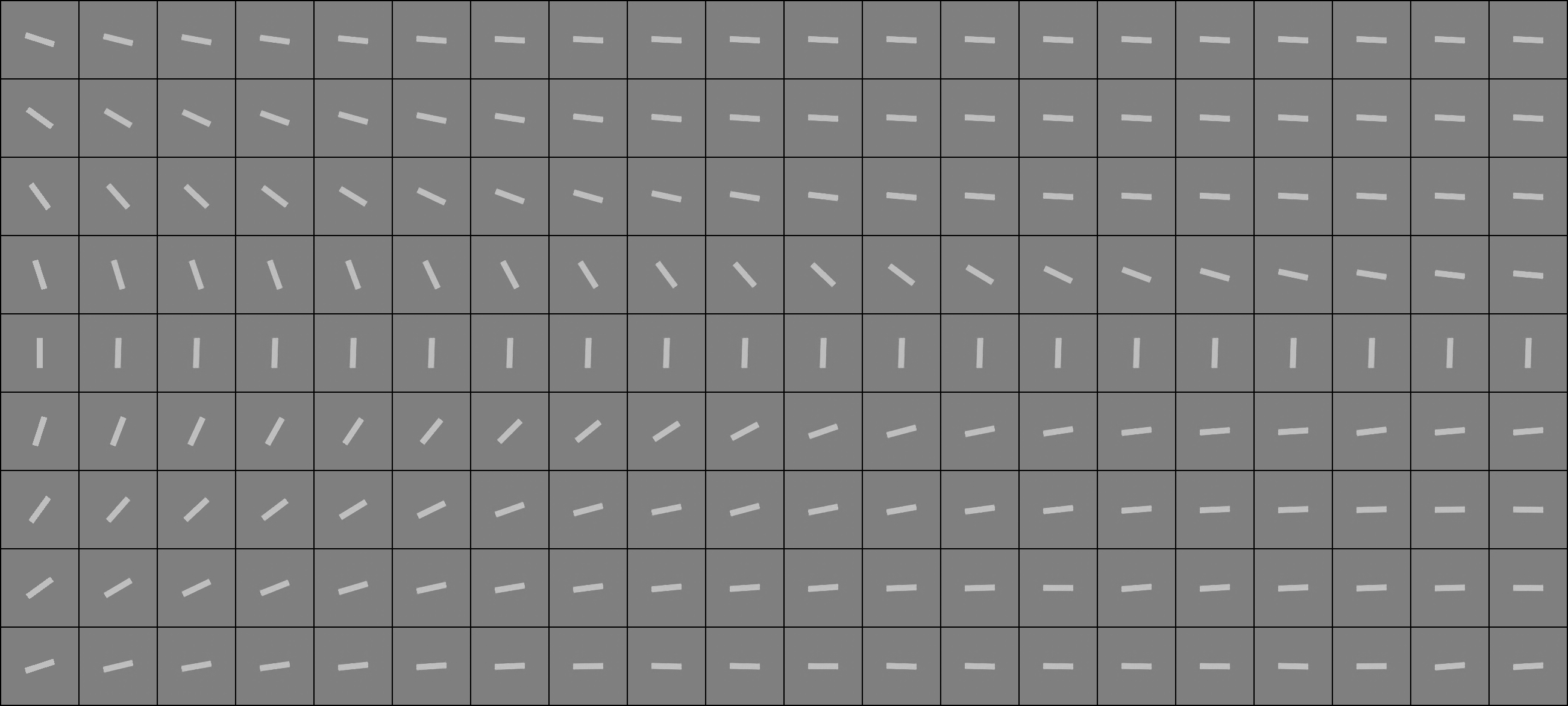}
\caption{Visualization of the trajectories $(x_n)_n$ for different initializations $x_0$ with two generators.
Left column: initialization, right column: reconstruction $x_{250}$, columns in between: images $x_n$ for $n$ approximately equispaced between $0$ and $250$.}
\label{fig:deblurring_two_charts}
\end{subfigure}
\caption{Gradient descent for the deblurring example.}
\end{figure}

First, we consider the inverse problem of noisy image deblurring. 
Here, the forward operator $\mathcal G$ in \eqref{IP} is linear and given by the convolution with a $30\times30$ 
Gaussian blur kernel with standard deviation $15$. 
In order to obtain outputs $y$ of the same size as the input $x$, we use constant padding with intensity $1/2$ within the convolution.
Moreover, the image is corrupted by white Gaussian noise $\eta$ with standard deviation $0.1$.
Given an observation $y$ generated by this degradation process, we aim to reconstruct the unknown ground truth image $x$.

\paragraph{Data set and Manifold Approximation.}

Here, we consider the data set of $128\times 128$ images showing a bright bar with a gray background that is centered and rotated. 
The intensity of fore- and background as well as the size of the bar are fixed.
Some example images from the data set are given in Figure~\ref{fig:dataset_deblurring}.
The data set forms a one-dimensional manifold parameterized by the rotation of the bar.
Therefore, it is homeomorphic to $S^1$ and does not admit a global parameterization since it contains a hole.

We approximate the data manifold by a mixture model of two VAEs and compare the result with the approximation with a single VAE, where the latent dimension is set to $d=1$. The learned charts are visualized in Figure~\ref{fig:bar_chart_1_gen} and \ref{fig:bar_chart_2_gen}.
We observe that the charts learned with a mixture of two VAEs can generate all possible rotations and overlap at their boundaries.
On the other hand, the chart learned with a single VAE does not cover all rotations but has a gap due to the injectivity of the decoder.
This gap is also represented in the final test loss of the model, which approximates the negative log likelihood of the test data.
It is given by $52.04$ for one generator and by $39.84$ for two generators.
Consequently, the model with two generators fits the data manifold much better.

\paragraph{Reconstruction.}

In order to reconstruct the ground truth image, we use our gradient descent scheme for the function \eqref{eq:objective_function} as outlined in Algorithm~\ref{alg:gd_manifold} for 500 iterations.
Since the function $F$ is defined on the whole $\R^{128\times128}$, we compute the Riemannian gradient $\nabla_{\mathcal M}F(x)$ accordingly to Remark~\ref{rem:riem_grad_ex}.
More precisely, for $x\in U_k$, we have $\nabla_{\mathcal M}F(x)= J (J^\tT J)^{-1} J^\tT\nabla F(x)$, where $\nabla F(x)$ is the Euclidean gradient of $F$ and $J=\nabla \mathcal D_k(\mathcal E_k(x))$ is the Jacobian of the $k$-th decoder evaluated at $\mathcal E_k(x)$.
Here, the Euclidean gradient $\nabla F(x)$ and the Jacobian matrix $J$ are computed by algorithmic differentiation.
Moreover, we use the retractions $\tilde R_{k,x}$ from \eqref{eq:retractions}.
As initialization $x_0$ of our gradient descent scheme, we use a random sample from the mixture of VAEs.
The results are visualized in Figure~\ref{fig:deblurring_reconstructions}.
We observe that the reconstructions with two generators always recover the ground truth images very well. 
On the other hand, the reconstructions with one generator often are unrealistic and do not match with the ground truth. 
These unrealistic images appear at exactly those points where the chart of the VAE with one generator does not cover the data manifold.

In order to better understand why the reconstructions with one generator often fail, we consider the trajectories $(x_n)_{n}$ generated by Algorithm~\ref{alg:gd_manifold} more in detail.
We consider a fixed ground truth image showing a horizontal bar and a corresponding observation as given in Figure~\ref{fig:traj_gt_obs}. Then, we run Algorithm~\ref{alg:gd_manifold} for different initializations. 
The results are given in Figure~\ref{fig:deblurring_one_chart} for one generator and in Figure~\ref{fig:deblurring_two_charts} for two generators. 
The left column shows the initialization $x_0$, and in the right column, there are the values $x_{250}$ after 250 gradient descent steps. 
The columns in between show the values $x_n$ for (approximately) equispaced $n$ between $0$ and $250$.
With two generators, the trajectory $(x_n)_n$ are a smooth transition from the initialization to the ground truth.
Only when the initialization is a vertical bar (middle row), the images $x_n$ remain similar to the initialization $x_0$ for all $n$, since this is a critical point of the $F|_{\mathcal M}$ and hence the Riemannian gradient is zero.
With one generator, we observe that some of the trajectories get stuck exactly at the gap, where the manifold is not covered by the chart.
At this point the latent representation of the corresponding image would have to jump, which is not possible. Therefore, a second generator is required here.

\subsection{Electrical Impedance Tomography}

\begin{figure}[t]
\begin{subfigure}[t]{\textwidth}
\includegraphics[width=\textwidth]{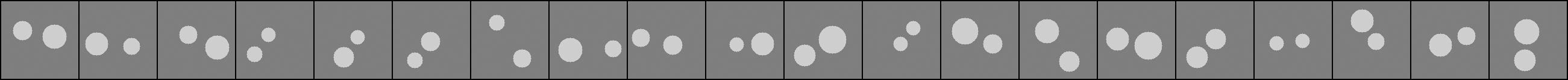}
\caption{Samples from the considered data set for the EIT example.}
\label{fig:dataset_EIT}
\end{subfigure}
\begin{subfigure}[t]{\textwidth}
\includegraphics[width=\textwidth]{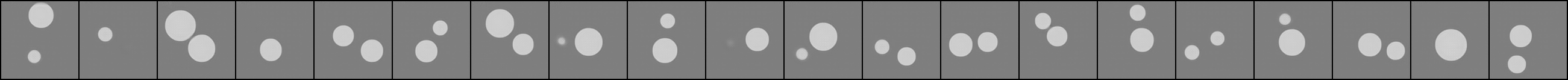}
\caption{Samples from the VAE with one generator.}
\label{fig:balls_chart_1_gen}
\end{subfigure}
\begin{subfigure}[t]{\textwidth}
\includegraphics[width=\textwidth]{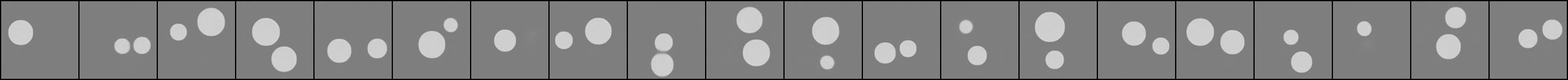}
\includegraphics[width=\textwidth]{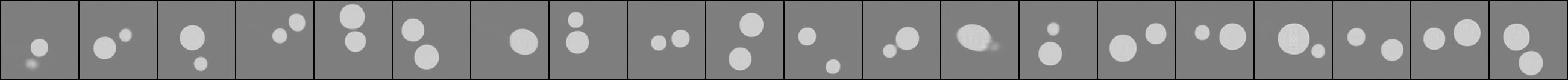}
\caption{Samples from the mixture of two VAEs. Top: first chart, bottom: second chart.}
\label{fig:balls_chart_2_gen}
\end{subfigure}
\begin{subfigure}[t]{\textwidth}
\includegraphics[width=\textwidth]{imgs/balls_ds_grid}
\includegraphics[width=\textwidth]{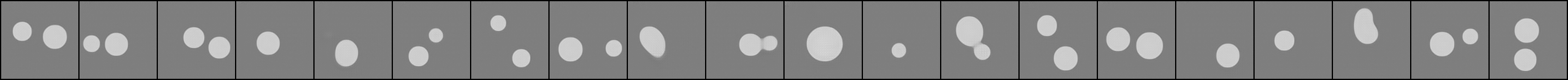}
\includegraphics[width=\textwidth]{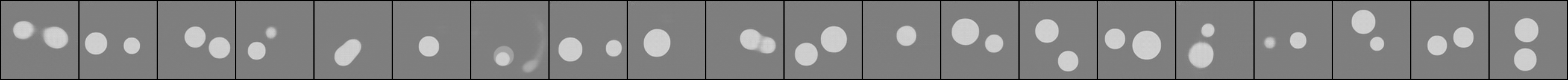}
\caption{Reconstructions. From top to bottom: ground truth image, reconstruction with one generator, reconstruction with two generators.}
\label{fig:balls_reconstructions}
\end{subfigure}
\caption{Data set, synthesized samples and reconstructions for the EIT example.}
\end{figure}

Finally, we consider the highly non-linear and ill-posed inverse problem of electrical impedance tomography (EIT, \citealp{cheney1999electrical}), which is also known in the mathematical literature as the Calder\'{o}n problem \citep{AP2006,FSU2019,MS2012}. 
EIT is a non-invasive, radiation-free method to measure the conductivity of a tissue through electrodes placed on the surface of the body. 
More precisely, electrical currents patterns are imposed on some of these electrodes and the resulting voltage differences are measured on the remaining ones. 
Although harmless, the use of this modality in practice is very limited because the standard reconstruction methods provide images with very low spatial resolution. This is an immediate consequence of the severe ill-posedness of the inverse problem \citep{alessandrini1988,mandache2001}.

Classical methods for solving this inverse problem include variational-type methods \citep{CINSG1990}, the Lagrangian method \citep{CZ1999}, the factorization method \citep{BH2000,KG2007}, the D-bar method \citep{SMI2000}, the enclosure method \citep{IS2000}, and the monotonicity method \citep{TR2002}.
Similarly as many other inverse problems,  deep learning methods have had a big impact on EIT.
For example, \citet{FY2020} propose an end-to-end neural network that learns the forward map $\mathcal G$ and its inverse.
Moreover, deep learning approaches can be combined with classical methods, e.g, by post processing methods \citep{HH2018,HHHK2019} or by variational learning algorithms \citep{SKJLH2019}.

\paragraph{Data set and Manifold Approximation.}
We consider the manifold consisting of $128\times 128$ images showing two bright non-overlapping balls with a gray background, representing conductivities with special inclusions.
The radius and the position of the balls vary, while the fore- and background intensities are fixed. 
Some exemplary samples of the data set are given in Figure~\ref{fig:dataset_EIT}.
\begin{remark}[Dimension and topology of the data manifold]
Since the balls are indistinguishable and not allowed to overlap, an image can be uniquely described by the angle between the two balls, the midpoint between both balls, their distance and the two radii.
Hence, the data manifold is homeomorphic to $\mathbb S^1\times (0,1)^2\times (0,1)\times (0,1)^2=\mathbb S^1\times (0,1)^5$. In particular, it contains a hole and does not admit a global parameterization.
\end{remark}

A slightly more general version of this manifold was considered by \citet{2023-alberti-santacesaria}, where Lipschitz stability is proven for a related inverse boundary value problem restricted to the manifold. Other types of inclusions (with unknown locations), notably polygonal and polyhedral inclusions, have been considered in the literature \citep{2021-beretta-francini-vessella,2022-beretta-francini,2022-aspri-beretta-francini-vessella}. The case of small inclusions is discussed by \citet{2004-ammari-kang}. 

We approximate the data manifold by a mixture of two VAEs and compare the results with the approximation with a single VAE. The latent dimension is set to the manifold dimension, i.e., $d=6$.
Some samples of the learned charts are given in Figure~\ref{fig:balls_chart_1_gen} and \ref{fig:balls_chart_2_gen}.
As in the previous example, both models produce mostly realistic samples.
The test loss is given by $365.21$ for one generator and by $229.99$ for two generators.
Since the test loss approximates the negative log likelihood value of the test data, this indicates that two generators are needed in order to cover the whole data manifold.

\paragraph{The Forward Operator and its Derivative.} From a mathematical viewpoint, EIT considers the following PDE with Neumann boundary conditions
\begin{equation}
\label{u_g}
\left \{ \begin{array}{rl}
- \nabla \cdot (\gamma \hspace{0.1cm} \nabla u_g) = 0& \text{in } \Omega, \\
\gamma \hspace{0.1cm} \partial_{\nu} u_g = g & \text{on } \partial \Omega,
\end{array}
\right.
\end{equation}
where $\Omega\subseteq\R^2$ is a bounded domain, $\gamma \in L^{\infty}(\Omega)$ is such that $\gamma(x) \geq \lambda > 0$ and $u_g\in H^1(\Omega)$ is the unique weak solution with zero boundary mean of \eqref{u_g} with Neumann boundary data \smash{$g \in H^{-\frac{1}{2}}_\diamond(\partial \Omega)$}, with $H^s_\diamond(\partial \Omega) = \{ f \in H^s(\partial \Omega) : \int_{\partial \Omega} f ds = 0 \}$.
From the physical point of view, $g$ represents the electric current applied on $\partial \Omega$ (through electrodes placed on the boundary of the body), $u_g$ is the electric potential and  $\gamma$ is the conductivity of the body in the whole domain $\Omega$. The inverse problem consists in the reconstruction of $\gamma$ from the knowledge of all pairs of boundary measurements $(g,u_g|_{\partial\Omega})$, namely, of all injected currents together with the corresponding electric voltages generated at the boundary. In a compact form, the measurements may be modelled by the Neumann-to-Dirichlet map 
\[
\begin{aligned}
\mathcal{G}(\gamma) \colon H^{-\frac{1}{2}}_\diamond(\partial \Omega) &\to H^{\frac{1}{2}}_\diamond(\partial \Omega)\\
g &\mapsto u_g \big|_{\partial \Omega}.
\end{aligned}
\]
Since the PDE \eqref{u_g} is linear,  the map $\mathcal{G}(\gamma) $ is linear. However, the forward map $\gamma\mapsto\mathcal{G}(\gamma) $ is nonlinear in $\gamma$, and so is the corresponding inverse problem.
The map $\mathcal{G}$ is continuously differentiable, and its Fr\'{e}chet derivative (see \citealp{H2019}) is given by $[\nabla \mathcal{G}(\gamma)](\sigma)(g) = w_g \big|_{\partial \Omega}$, where $w_g \in H^1(\Omega)$ is the unique weak solution with zero boundary mean of 
\begin{equation*}
\left \{ \begin{array}{rl}
- \nabla \cdot (\gamma \hspace{0.1cm} \nabla w_g) = \nabla \cdot (\sigma \hspace{0.1cm} \nabla u_g)& \text{in } \Omega, \\
- \gamma \hspace{0.1cm} \partial_{\nu} w_g = \sigma \hspace{0.1cm} \partial_{\nu} u_g & \text{on } \partial \Omega,
\end{array}
\right.
\end{equation*}
where $u_g \in H^1(\Omega)$ is the unique weak solution with zero boundary mean that solves \eqref{u_g}. We included the expression of this derivative in the continuous setting for completeness, but, as a matter of fact, we will need only its semi-discrete counterpart given below.

\paragraph{Discretization and Objective Function.}
In our implementations, we discretize the linear mappings $G(\gamma)$ by restricting them to a finite dimensional subspace spanned by 
a-priori fixed boundary functions $g_1,\dots,g_N\in H^{-\frac{1}{2}}_\diamond(\partial \Omega)$.
Then, following \cite[eqt.~(2.2)]{BMPS2018}, we reconstruct the conductivity  by minimizing the semi-discrete functional 
\begin{equation}
\label{J}
    F(\gamma) = \frac{1}{2} \sum_{n = 1}^N \int_{\partial \Omega} | u_{g_n}(s) - (u_{\text{true}})_{g_n}(s)|^2 ds,
\end{equation} 
where $(u_{\text{true}})_{g_n}$ is the observed data.
In our discrete setting, we represent the conductivitiy $\gamma$ by a piecewise constant function $\gamma = \sum_{m=1}^M \gamma_m \mathbbm{1}_{T_m}$ on a triangulation $(T_m)_{m=1,\dots,M}$. Then, following Equation (2.20) from \citet{BMPS2018}, the derivative of \eqref{J} with respect to $\gamma$ is given by \begin{equation}\label{eq:derivative_F_EIT}
\frac{dF}{d\gamma_m}(\gamma) = \sum_{n=1}^N \int_{T_m} \nabla u_{g_n}(x) \cdot \nabla z_{g_n}(x) dx,  
\end{equation}
where $z_{g_n}$ solves 
\begin{equation}
\label{z_g_n}
  \left \{ \begin{array}{rl}
- \nabla \cdot (\gamma \hspace{0.1cm} \nabla z_{g_n}) = 0\phantom{(u_{\text{true}})_{g_n} - u_{g_n}}& \text{in } \Omega, \\
\gamma \hspace{0.1cm} \partial_{\nu} z_{g_n} = (u_{\text{true}})_{g_n} - u_{g_n}\phantom{0}& \text{on } \partial \Omega,
\end{array}
\right.  
\end{equation}
with the normalization $\int_{\partial \Omega} z_{g_n}(s) ds = \int_{\partial \Omega} (u_{\text{true}})_{g_n}(s) ds$.

\paragraph{Implementation Details.} In our experiments the domain $\Omega$ is given by the unit square $[0,1]^2$. For solving the PDEs \eqref{u_g} and \eqref{z_g_n}, we use a finite element solver from the DOLFIN library \citep{LW2010}. 
We employ meshes that are coarser in the middle of $\Omega$ and finer close to the boundary.
To simulate the approximation error of the meshes, and to avoid inverse crimes, we use a fine mesh to generate the observation and a coarser one for the reconstructions.
We use $N=15$ boundary functions, which are chosen as follows.
We divide each of the four edges of the unit square $[0,1]^2$ into 4 segments and denote by $b_1,\dots,b_{16}$ the functions that are equal to 
$1$ on one of these segments and $0$ otherwise. 
Then, we define the boundary functions as $g_n=\sum_{i=1}^{16}a_{n,i}b_i$, where the matrix $A=(a_{n,i})_{n=1,\dots,15,i=1,\dots,16}$ is the $16\times16$ Haar matrix without the first row.
More precisely, the rows of $A$ are given by the rows of the 
matrices $2^{-k/2}(\mathrm{Id}_{2^{4-k}}\otimes (1,-1) \otimes e_{2^{k-1}}^\tT)$ for $k=1,\dots,4$, where $\otimes$ is the Kronecker product and $e_j\in\R^j$ is the vector where all entries are $1$.

\paragraph{Results.}

We reconstruct the ground truth images from the observations by minimizing the functional $F$ from \eqref{J} subject to $\gamma\in\mathcal M$.
To this end, we apply the gradient descent scheme from Algorithm~\ref{alg:gd_manifold} for 100 steps. Since the evaluation of the forward operator and 
its derivative include the numerical solution of a PDE, it is computationally very costly. Hence, we aim to use as few iterations of Algorithm~\ref{alg:gd_manifold} as possible. To this end, we apply the adaptive step size scheme from Algorithm~\ref{alg:adaptive_steps}.
As retractions we use $\tilde R_{k,x}$ from \eqref{eq:retractions}.
The initialization $\gamma_0$ of the gradient descent scheme is given by a random sample from the mixture of VAEs.

Since $F$ is defined on the whole $\R_+^{128\times128}$, we use again Remark~\ref{rem:riem_grad_ex} for the evaluation of the Riemannian gradient.
More precisely, for $\gamma\in U_k$, we have that $\nabla_{\mathcal M} F(\gamma)=J(J^\tT J)^{-1}J^\tT\nabla F(\gamma)$, where $\nabla F(\gamma)$ is the Euclidean gradient of $F$ and $J=\nabla \mathcal D_k(\mathcal E_k(\gamma))$.  
Here, we compute $\nabla F(\gamma)$  by \eqref{eq:derivative_F_EIT} and $J$ by algorithmic differentiation.

The reconstructions for 20 different ground truths are visualized in Figure~\ref{fig:balls_reconstructions}.
We observe that both models capture the ground truth structure in most cases, but also fail sometimes. 
Nevertheless, the reconstructions with the mixture of two VAEs recover the correct structure more often and more accurately than the single VAE, which can be explained by the better coverage of the data manifold.
To quantify the difference more in detail, we rerun the experiment with $200$ different ground truth images and compare the results with one and two generators using the PSNR and SSIM.
As an additional evaluation metric, we run the segmentation algorithm proposed by \citet{O1979} on the ground truth and reconstruction image and compare the resulting segmentation masks using the SSIM.
The resulting values are given in the following table.
\begin{center}
\begin{tabular}{c|cc}
&One generator&Two generators\\\hline
PSNR&$23.64\pm3.91$&$24.76\pm3.79$\\
SSIM&$0.8951\pm 0.0377$&$0.9111\pm 0.0368$\\
segment+SSIM&$0.8498\pm0.0626$&$0.8667\pm0.0614$
\end{tabular}
\end{center}
Consequently, the reconstructions with two generators are significantly better than those with one generator for all evaluation metrics.

\section{Conclusions}\label{sec:conclusions}

In this paper we introduced mixture models of VAEs for learning manifolds of arbitrary topology.
The corresponding decoders and encoders of the VAEs provide analytic access to the resulting charts
and are learned by a loss function that approximates the negative log-likelihood function. 
For minimizing functions $F$ defined on the learned manifold we proposed a Riemannian gradient descent scheme.
In the case of inverse problems, $F$ is chosen as a data-fidelity term.
Finally, we demonstrated the advantages of using several generators on numerical examples.

This work can be extended in several directions.
First, gradient descent methods converge only locally and are not necessarily fast. Therefore,
it would be interesting to extend the minimization of the functional $F$ in Section~\ref{opt_prob} to higher-order methods 
or incorporate momentum parameters.
Moreover, a careful choice of the initialization could improve the convergence behavior.
Further, our reconstruction method could be extended to Bayesian inverse problems.
Since the mixture model of VAEs provides us with a probability distribution and an (approximate) density,
 stochastic sampling methods like the Langevin dynamics could be used for quantifying uncertainties within our reconstructions.
Indeed, Langevin dynamics on Riemannian manifolds are still an active area of research.
Further, for large numbers $K$ of charts the mixture of VAEs might have a considerable number of parameters.
As a remedy, we could incorporate the selection of the chart as conditioning parameter in one conditional decoder-encoder pair (see \citealp{SLY2015} as a reference for conditional VAEs).
Finally, recent papers show that diffusion models provide an implicit representation of the data manifold \citep{BSS2022,RLCC2022}.
It would be interesting to investigate optimization models on such manifolds in order to apply them to inverse problems.


\acks{This material is based upon work supported by the Air Force Office of Scientific Research under award numbers FA8655-20-1-7027 and FA8655-23-1-7083. We acknowledge the support of Fondazione Compagnia di San Paolo. Co-funded by the European Union (ERC, SAMPDE, 101041040). Views and opinions expressed are however those of the authors only and do not necessarily reflect those of the European Union or the European Research Council. Neither the European Union nor the granting authority can be held responsible for them. GSA, MS and SS are members of the ``Gruppo Nazionale per l’Analisi Matematica, la Probabilità e le loro Applicazioni'', of the ``Istituto Nazionale di Alta Matematica''. The research of GSA, MS and SS was supported in part by the MUR Excellence Department Project awarded to the Department of Mathematics, University of Genoa, CUP D33C23001110001.
JH acknowledges funding by the German Research Foundation (DFG) within the project STE 571/16-1 and by the EPSRC programme grant ``The Mathematics of Deep Learning'' with reference EP/V026259/1.}

\appendix

\section{Derivation of the ELBO}\label{app:ELBO}

By Jensen's inequality the evidence can be lower-bounded by
\begin{equation}\label{eq:long_ELBO_derivation}
\begin{aligned}
\log(p_{\tilde X}(x))&=\log\Big(\int_{\R^d} p_{Z,\tilde X}(z,x) \dx z\Big)
\\&=\log\Big(\int_{\R^d} \frac{p_{Z,\tilde X}(z,x)}{p_{\tilde Z|X=x}(z)}p_{\tilde Z|X=x}(z) \dx z\Big)\\
&\geq \int_{\R^d} \log\Big(\frac{p_{Z,\tilde X}(z,x)}{p_{\tilde Z|X=x}(z)}\Big)p_{\tilde Z|X=x}(z) \dx z\Big)\\
&=\E_{z\sim P_{\tilde Z|X=x}}\Big[\log\Big(\frac{p_{Z}(z)p_{\tilde X|Z=z}(x)}{p_{\tilde Z|X=x}(z)}\Big)\Big]\\
&=\E_{z\sim P_{\tilde Z|X=x}}[\log(p_Z(z))+\log(p_{\tilde X|Z=z}(x))-\log(p_{\tilde Z|X=x}(z))].
\end{aligned}
\end{equation}
Accordingly to the definition of $\tilde Z$ and $\tilde X$, we have that $p_{\tilde X|Z=z}(x)=\mathcal N(x;D(z),\sigma_x^2 I_n)$ and $p_{\tilde Z|X=x}(z)=\mathcal N(z;E(x),\sigma_z^2 I_d)$.
Thus, the above formula is, up to a constant, equal to
\begin{align}
\E_{z\sim P_{\tilde Z|X=x}}[\log(p_Z(z))-\tfrac1{2\sigma_x^2}\|x-D(z)\|^2-\log(\mathcal N(z;E(x),\sigma_z^2 I_d))].
\end{align}
Considering the substitution $\xi=(z-E(x))/\sigma_z$, we obtain
\begin{align}
\E_{\xi\sim\mathcal N(0,I_d)}[\log(p_Z(E(x)+\sigma_z\xi))-\tfrac1{2\sigma_x^2}\|D(E(x)+\sigma_z\xi)-x\|^2
-\log(\mathcal N(\xi;0,I_d))].
\end{align}
Note that also the last summand does not depend on $D$ and $E$. Thus, we obtain, up to a constant, the \emph{evidence lower bound} (ELBO) given by
\begin{equation}
\ELBO(x|\theta)\coloneqq \E_{\xi\sim\mathcal N(0,I_d)}[\log(p_Z(E(x)+\sigma_z\xi))-\tfrac1{2\sigma_x^2}\|D(E(x)+\sigma_z\xi)-x\|^2].
\end{equation}

\section{Error Bound for the Approximation of $\beta_{ik}$ by $\tilde \beta_{ik}$}\label{app:Lipschitz_bound}

It is well-known that the difference between evidence and ELBO can be expressed in terms of the Kullback-Leibler divergence (see \citealp[Sec 2.2]{KW2019}). In the special case that $E\circ D=\mathrm{Id}$, which is relevant in this paper, this estimate can be simplified by the following lemma. To this end, denote by
$$
\mathcal F(x|\theta)\coloneqq\E_{z\sim P_{\tilde Z|X=x}}[\log(p_Z(z))+\log(p_{\tilde X|Z=z}(x))-\log(p_{\tilde Z|X=x}(z))]=\ELBO(x|\theta)+\mathrm{const}
$$
the ELBO before leaving out the constants.

\begin{lemma}\label{lem:ELBO_evidence_KL}
Assume that $E\circ D=\mathrm{Id}$. Then, it holds
$$
\log(p_{\tilde X}(x))-\mathcal F(x|\theta)=\mathrm{KL}(\mathcal N(E(x),\sigma_z^2I_d),E_\#\mathcal N(x,\sigma_x^2I_n)).
$$
\end{lemma}
\begin{proof}
By Equation (2.8) in \citet{KW2019}, it holds that
$$
\log(p_{\tilde X}(x))-\mathcal F(x|\theta)=\mathrm{KL}(P_{\tilde Z|X=x},P_{Z|\tilde X=x}).
$$
Inserting the definitions $\tilde Z=X+\xi$ and $Z=E(D(Z))=E(D(Z)+\eta-\eta)=E(\tilde X-\eta)$, this is equal to
$$
\mathrm{KL}(P_{E(X)+\xi|X=x},P_{E(\tilde X-\eta)|\tilde X=x})=\mathrm{KL}(P_{E(x)+\xi},P_{E(x-\eta)})=\mathrm{KL}(P_{E(x)+\xi},E_\#P_{x-\eta}).
$$
Using $\xi\sim\mathcal N(0,\sigma_z^2 I_d)$ and $\eta=\mathcal N(0,\sigma_x^2I_n)$, we arrive at the assertion.
\end{proof}

The next two lemmas exploit this estimate for bounding the approximation error between $\tilde \beta_{ik}$ and $\beta_{ik}$.

\begin{lemma}
Assume that $\exp(\mathrm{KL}(\mathcal N(E_k(x_i),\sigma_z^2I_d),{E_k}_\#\mathcal N(x_i,\sigma_x^2I_n)))\leq L$ for all $x_i$, $k$. Then
$$
\frac1L\beta_{ik}\leq\tilde\beta_{ik}\leq L\beta_{ik},
$$
where $\beta_{ik}$ and $\tilde \beta_{ik}$ are defined in \eqref{eq_true_betas} and \eqref{eq_approx_betas}.
\end{lemma}
\begin{proof}
Due to Lemma~\ref{lem:ELBO_evidence_KL}, we have that
$$
1\leq\frac{p_{\tilde X_k}(x_i)}{\exp(\mathcal F(x_i|\theta_k))}\leq\exp(\mathrm{KL}(\mathcal N(E_k(x_i),\sigma_z^2I_d),{E_k}_\#\mathcal N(x_i,\sigma_x^2I_n)))\leq L,
$$
i.e.,
$$
\frac1L p_{\tilde X_k}(x_i)\leq\exp(\mathcal F(x_i|\theta_k))\leq p_{\tilde X_k}(x_i).
$$
Since by $\ELBO(x_i|\theta_k)=\mathcal F(x_i|\theta_k)+\mathrm{const}$, it holds that $\tilde \beta_{ik}=\tfrac{\alpha_k \exp(\mathcal F(x_i|\theta_k))}{\sum_{j=1}^K \alpha_j \exp(\mathcal F(x_i|\theta_j))}$, this implies that
$$
\frac1L\beta_{ik}=\tfrac{\frac1L\alpha_k p_{\tilde X_k}(x_i)}{\sum_{j=1}^K \alpha_j p_{\tilde X_k}(x_i)}\leq\underbrace{\tfrac{\alpha_k \exp(\mathcal F(x_i|\theta_k))}{\sum_{j=1}^K \alpha_j \exp(\mathcal F(x_i|\theta_j))}}_{=\tilde\beta_{ik}}\leq \tfrac{\alpha_k p_{\tilde X_k}(x_i)}{\frac1L\sum_{j=1}^K \alpha_j p_{\tilde X_k}(x_i)}=L\beta_{ik},
$$
which concludes the proof.
\end{proof}

\begin{lemma}
Let $E=A\circ T$ where $A\colon\R^n\to\R^d$ is given by the matrix $A=(I_d|0)$ and $T$ is invertible and bi-Lipschitz with Lipschitz constants $L_1$ and $L_2$ for $T$ and $T^{-1}$.
Then
$$
\mathrm{KL}(\mathcal N(E(x),\sigma_z^2I_d),{E}_\#\mathcal N(x,\sigma_x^2I_n))\leq \log(L_1^n L_2^n)+\frac{d}2(\tfrac{L_2^2\sigma_z^2}{\sigma_x^2}-1-\log(\tfrac{L_2^2\sigma_z^2}{\sigma_x^2})).
$$
\end{lemma}
\begin{proof}
We estimate the density $p_{E_\#\mathcal N(x,\sigma_x^2I_n))}(z)$ from below.
By \cite[Lemma 4]{ADHHMS2022}, we have that 
$$
p_{T_\#\mathcal N(x,\sigma_x^2I_n)}\geq \frac{1}{L_1^nL_2^n}\mathcal N(T(x),\frac{\sigma_x}{L_2^2}I_n).
$$
Using the projection property of Gaussian distributions, this implies that $E_\#\mathcal N(x,\sigma_x^2I_n))=A_\#(T_\#\mathcal N(x,\sigma_x^2I_n))$ fulfills
$$
p_{E_\#\mathcal N(x,\sigma_x^2I_n))}\geq\frac{1}{L_1^nL_2^n}\mathcal N(E(x),\frac{\sigma_x}{L_2^2}I_d).
$$
Hence, by the definition of the Kullback-Leibler divergence, it holds
\begin{align}
&\quad\mathrm{KL}(\mathcal N(E(x),\sigma_z^2I_d),{E}_\#\mathcal N(x,\sigma_x^2I_n))\\
&=\E_{z\sim\mathcal N(E(x),\sigma_z^2I_d)}\Big[\log\Big(\frac{\mathcal N(z|E(x),\sigma_z^2I_d)}{p_{E_\#\mathcal N(x,\sigma_x^2I_n))}(z)}\Big)\Big]\\
&\leq\log(L_1^n L_2^n)+\E_{z\sim\mathcal N(E(x),\sigma_z^2I_d)}\Big[\log\Big(\frac{\mathcal N(z|E(x),\sigma_z^2I_d)}{\mathcal N(z|E(x),\frac{\sigma_x^2}{L_2^2}I_d)}\Big)\Big]\\
&=\log(L_1^n L_2^n)+\mathrm{KL}(\mathcal N(E(x),\sigma_z^2I_d),\mathcal N(z|E(x),\frac{\sigma_x^2}{L_2^2}I_d)).
\end{align}
Inserting the formula for the KL divergence between two normal distributions, we obtain
\begin{align}
\mathrm{KL}(\mathcal N(E(x),\sigma_z^2I_d),\mathcal N(z|E(x),\frac{\sigma_x^2}{L_2^2}I_d))&=\frac12\big(\mathrm{trace}(\tfrac{L_2^2\sigma_z^2}{\sigma_x^2} I_d)-d+d\log(\tfrac{\sigma_x^2}{L_2^2\sigma_z^2})\big)\\
&=\frac{d}2(\tfrac{L_2^2\sigma_z^2}{\sigma_x^2}-1-\log(\tfrac{L_2^2\sigma_z^2}{\sigma_x^2})),
\end{align}
which proves the claim.
\end{proof}

Combining the previous two lemmas, we obtain the following corollary.

\begin{corollary}\label{cor:Lipschitz_bound}
Assume that $E_1,...,E_K$ are of the form $E_k=A\circ T_k$, where $A\colon\R^n\to\R^d$ is given by the matrix $A=(I_d|0)$ and $T_k$ is invertible and bi-Lipschitz with Lipschitz constants $L_1$ and $L_2$ for $T$ and $T^{-1}$ and assume that $E_k\circ D_k=\mathrm{Id}$. Then
$$
\frac1L\beta_{ik}\leq\tilde\beta_{ik}\leq L\beta_{ik},
$$
where $\beta_{ik}$ and $\tilde \beta_{ik}$ are given in \eqref{eq_true_betas} and \eqref{eq_approx_betas} and $L=L_1^nL_2^n\exp(\tfrac{d}2(\tfrac{L_2^2\sigma_z^2}{\sigma_x^2}-1-\log(\tfrac{L_2^2\sigma_z^2}{\sigma_x^2})))$.
\end{corollary}

Note  in particular that $\tilde \beta_{ik}\to0$ whenever $\beta_{ik}\to 0$ and, by
$$
\tilde\beta_{ik}=1-\sum_{\substack{l=1\\l\neq k}}^K\tilde\beta_{ik}\geq1-L\sum_{\substack{l=1\\l\neq k}}^K\beta_{ik}=1-L(1-\beta_{ik})=1-L+L\beta_{ik}
$$
that $\tilde\beta_{ik}\to 1$ whenever $\beta_{ik}\to 1$.
However, the constant $L$ from the corollary depends exponentially on the dimension $n$, which might limit its applicability when $n$ is very large.

\vskip 0.2in
\bibliography{ref}

\end{document}